\documentclass{article}

\usepackage[preprint]{neurips_2026}

\usepackage[utf8]{inputenc}
\usepackage[T1]{fontenc}

\usepackage{amsmath}
\usepackage{amssymb}
\usepackage{mathtools}
\usepackage{amsthm}
\usepackage{bm}

\usepackage{booktabs}
\usepackage{multirow}
\usepackage{array}
\usepackage{xcolor}
\usepackage{graphicx}
\usepackage{float}
\usepackage{tikz}
\usetikzlibrary{arrows.meta,positioning,calc,fit,backgrounds,shapes.geometric}
\usepackage{pgfplots}
\pgfplotsset{compat=1.17}

\usepackage{nicefrac}
\usepackage{microtype}
\usepackage{enumitem}
\usepackage{thmtools}
\usepackage[framemethod=TikZ]{mdframed}
\usepackage{url}
\usepackage{hyperref}
\hypersetup{colorlinks=true, linkcolor=blue!55!black, citecolor=blue!55!black,
            urlcolor=blue!55!black}

\definecolor{thmMain}{RGB}{40,72,120}     
\definecolor{thmMainBg}{RGB}{234,240,250}
\definecolor{thmDef}{RGB}{38,104,86}       
\definecolor{thmDefBg}{RGB}{233,244,240}
\definecolor{thmKey}{RGB}{150,84,30}       
\definecolor{thmKeyBg}{RGB}{249,242,232}
\definecolor{thmPlain}{RGB}{70,70,78}      

\newtheoremstyle{styMain}{8pt}{8pt}{}{}{\bfseries\color{thmMain}}{.}{.5em}{}
\newtheoremstyle{styDef}{7pt}{7pt}{}{}{\bfseries\color{thmDef}}{.}{.5em}{}
\newtheoremstyle{styKey}{7pt}{7pt}{}{}{\bfseries\color{thmKey}}{.}{.5em}{}
\newtheoremstyle{styPlain}{7pt}{4pt}{}{}{\bfseries\color{thmPlain}}{.}{.5em}{}

\theoremstyle{styMain}\newtheorem{thminner}{Theorem}
\theoremstyle{styDef}\newtheorem{definner}{Definition}
\theoremstyle{styPlain}
\newtheorem{propinner}{Proposition}
\newtheorem{corinner}{Corollary}
\newtheorem{assuinner}{Assumption}
\theoremstyle{styKey}\newtheorem{keyinner}[propinner]{Proposition}

\mdfdefinestyle{mdMain}{linewidth=3pt, linecolor=thmMain, backgroundcolor=thmMainBg,
  topline=false, bottomline=false, rightline=false, leftline=true, roundcorner=2pt,
  innertopmargin=6pt, innerbottommargin=6pt, innerleftmargin=8pt, innerrightmargin=8pt,
  skipabove=8pt, skipbelow=8pt, nobreak=true}
\mdfdefinestyle{mdDef}{linewidth=2.5pt, linecolor=thmDef, backgroundcolor=thmDefBg,
  topline=false, bottomline=false, rightline=false, leftline=true, roundcorner=1pt,
  innertopmargin=5pt, innerbottommargin=5pt, innerleftmargin=7pt, innerrightmargin=7pt,
  skipabove=7pt, skipbelow=7pt, nobreak=true}
\mdfdefinestyle{mdKey}{linewidth=2.5pt, linecolor=thmKey, backgroundcolor=thmKeyBg,
  topline=false, bottomline=false, rightline=false, leftline=true, roundcorner=1pt,
  innertopmargin=5pt, innerbottommargin=5pt, innerleftmargin=7pt, innerrightmargin=7pt,
  skipabove=7pt, skipbelow=7pt, nobreak=true}
\mdfdefinestyle{mdPlain}{linewidth=2pt, linecolor=thmPlain!75, backgroundcolor=black!3,
  topline=false, bottomline=false, rightline=false, leftline=true, roundcorner=1pt,
  innertopmargin=5pt, innerbottommargin=5pt, innerleftmargin=7pt, innerrightmargin=7pt,
  skipabove=6pt, skipbelow=6pt, nobreak=true}

\newenvironment{theorem}{\begin{mdframed}[style=mdMain]\begin{thminner}}
  {\end{thminner}\end{mdframed}}
\newenvironment{definition}{\begin{mdframed}[style=mdDef]\begin{definner}}
  {\end{definner}\end{mdframed}}

\theoremstyle{styMain}\newtheorem{thmplaininner}[thminner]{Theorem}

\newenvironment{proposition}{\begin{mdframed}[style=mdPlain]\begin{propinner}}
  {\end{propinner}\end{mdframed}}
\newenvironment{corollary}{\begin{mdframed}[style=mdPlain]\begin{corinner}}
  {\end{corinner}\end{mdframed}}

\theoremstyle{remark}
\newtheorem{remark}{Remark}
\newtheorem*{prediction*}{Prediction}

\makeatletter
\renewenvironment{proof}[1][\proofname]{\par
  \pushQED{\qed}%
  \normalfont \topsep6\p@\@plus6\p@\relax
  \trivlist
  \item[\hskip\labelsep\bfseries\itshape #1\@addpunct{.}]\ignorespaces
}{%
  \popQED\endtrivlist\@endpefalse
}
\makeatother

\theoremstyle{styMain}\newtheorem{athminner}{Theorem}[section]
\theoremstyle{styPlain}
\newtheorem{apropinner}[athminner]{Proposition}
\newtheorem{acorinner}[athminner]{Corollary}
\newtheorem{aleminner}[athminner]{Lemma}
\newenvironment{atheorem}{\begin{mdframed}[style=mdMain]\begin{athminner}}
  {\end{athminner}\end{mdframed}}
\newenvironment{aproposition}{\begin{mdframed}[style=mdPlain]\begin{apropinner}}
  {\end{apropinner}\end{mdframed}}
\newenvironment{acorollary}{\begin{mdframed}[style=mdPlain]\begin{acorinner}}
  {\end{acorinner}\end{mdframed}}
\newenvironment{alemma}{\begin{mdframed}[style=mdPlain]\begin{aleminner}}
  {\end{aleminner}\end{mdframed}}

\newcommand{\R}{\mathbb{R}}
\newcommand{\Z}{Z}
\newcommand{\z}{z}
\newcommand{\E}{\mathbb{E}}
\newcommand{\KL}[2]{D_{\mathrm{KL}}\!\left[#1\,\middle\|\,#2\right]}
\newcommand{\Norm}[2]{\mathcal{N}\!\left(#1,#2\right)}
\newcommand{\Sig}{\mathrm{SIGReg}}
\newcommand{\hVR}{\hat{H}_{\mathrm{VR}}}

\newcommand{\hPD}{\hat{H}_{\mathrm{PD}}}
\newcommand{\hN}{h_{\mathcal{N}}}
\newcommand{\hstar}{h^{\star}}
\newcommand{\Istar}{I^{\star}}
\newcommand{\Ce}{C_{\epsilon}}
\newcommand{\Cgoal}{C_{\mathrm{goal}}}
\newcommand{\Cunc}{C_{\mathrm{unc}}}
\newcommand{\rhoS}{\rho_{S}}
\newcommand{\Fsig}{F_{\mathrm{SIG}\text{-}\mathrm{mod}}}
\newcommand{\Fvr}{F_{\mathrm{VR}\text{-}\mathrm{mod}}}
\newcommand{\Lpred}{\mathcal{L}_{\mathrm{pred}}}

\newcommand{\tr}{\operatorname{tr}}
\newcommand{\Cov}{\operatorname{Cov}}
\newcommand{\Var}{\operatorname{Var}}

\title{The SIGReg Objective as Variational Free Energy:\\
A Theoretical Active-Inference Account of JEPA World Models}

\author{%
  Fabio Arnez\thanks{Corresponding author; comments on this manuscript are welcome.
    This paper establishes a formal correspondence between the SIGReg objective and
    the Active Inference framework. Empirical validation of its predictions is left
    to separate work.} \\
  Universit\'e Paris-Saclay, CEA, List\\
  F-91120, Palaiseau, France \\
  fabio.arnez@cea.fr
  \And
  Alexandra Gomez-Villa \\
  Computer Vision Center\\
  Barcelona, Spain \\
  agomezvi@cvc.uab.es
}

\begin{document}

\maketitle

\begin{abstract}
Joint-Embedding Predictive Architectures (JEPAs) are the dominant design for
latent world models, yet they are usually justified by empirical performance
rather than a normative principle. We show that the choice of anti-collapse
regulariser determines whether a JEPA's training objective, a prediction loss
plus a weighted embedding regulariser, is a valid Active Inference (AIF)
variational free energy. We organise four non-contrastive regularisers (VICReg,
LogDet, PairDist, and SIGReg) into an \emph{entropy-estimator hierarchy} indexed
by a \emph{prior-miscalibration gap}, and show that the gap's \emph{sign},
whether the estimator bounds the latent entropy from above or below, decides
whether the AIF surprise bound survives: VICReg and LogDet are unsafe upper
bounds, PairDist a safe lower bound, and SIGReg eliminates the gap. We then prove
a correspondence theorem: under the standard constant-noise encoder model and
successful SIGReg enforcement (isotropic-Gaussian embeddings), the gap vanishes,
the objective becomes an \emph{exact} information bottleneck, the surprise bound
is preserved, and the latent goal cost becomes an exact proxy for AIF pragmatic
value, whereas VICReg leaves an irreducible second-order anisotropy term. We
extend the correspondence to multi-step expected free energy, ensemble epistemic
value, and a learned-policy regime, and we identify the one AIF term no current
JEPA world model computes: the \emph{state-epistemic value}, a future-state
coverage signal. The predictions differ in \emph{kind}, not degree, and are
stated here as theoretical consequences left for empirical test in separate
work; full proofs are in Appendix~\ref{app:proofs}, and the algebraic core of
every result is machine-verified in Lean~4 (Appendix~\ref{app:lean}).
\end{abstract}

\section{Introduction}
\label{sec:intro}

Active Inference (AIF) and Joint-Embedding Predictive Architectures (JEPA) are
two independently developed accounts of how an agent can learn and act from
high-dimensional observations. AIF, rooted in the Free Energy Principle
\citep{friston2010free}, is normative: the agent maintains a generative model
and acts to minimise variational free energy, an upper bound on sensory
surprise. JEPA, developed within the self-supervised learning (SSL) tradition
\citep{lecun2022path}, is architectural: a deterministic encoder maps
observations to latent embeddings and a predictor models the dynamics in that
latent space, avoiding pixel-level reconstruction.

Despite their separate origins, both frameworks factor the learning objective
into the same two competing terms: a \emph{complexity} term that penalises
departures from the dynamics model and an \emph{informativeness} term that
rewards capturing useful structure in observations. The Information Bottleneck
(IB) principle \citep{tishby2000information} formalises this trade-off as a
Lagrangian $\mathcal{L}_{\mathrm{IB}} = \mathcal{D}[\text{encoder}\,\|\,
\text{prior}] - \hat{I}(\Z;X)$, where the divergence $\mathcal{D}$ and the
information estimator $\hat I$ vary across implementations. In AIF the
informativeness term is a likelihood or a mutual-information bound
\citep{mazzaglia2021contrastive,mazzaglia2022free}. In JEPA world models it is
maintained implicitly by an anti-collapse regulariser such as VICReg
\citep{bardes2022vicreg}. The shared structure invites a precise question:
\emph{how exactly do the terms of a JEPA world model's training loss correspond
to the terms of the AIF variational free energy?}

\paragraph{The problem: non-contrastive informativeness is approximate.}
The answer hinges on how the informativeness term is maintained. The contrastive
route \citep{mazzaglia2021contrastive} estimates mutual information with an
InfoNCE \emph{lower} bound \citep{oord2018representation}, so the agent never
overstates its own informativeness and the free-energy bound is preserved by
construction, at the cost of stochastic encoders and a learned critic. Every
deployed JEPA world model
\citep{sobal2025closing,zhou2025dinowm,assran2025vjepa2,maes2026leworldmodel}
instead uses a deterministic encoder with a non-contrastive regulariser. Under
the standard constant-noise model \citep{kolchinsky2019caveats}, maximising
informativeness then reduces to maximising the marginal differential entropy
$h(\Z)$, which is intractable in high dimension. VICReg approximates it through
the first two moments \citep{shwartzziv2023information}, introducing a gap
between the true entropy and its proxy. Crucially, this gap is not merely a
matter of degree: VICReg's proxy is an \emph{upper} bound on the Gaussian
reference entropy (via Hadamard's inequality), which is itself an upper bound on
$h(\Z)$ (via the maximum-entropy theorem). Maximising an upper bound provides no
guarantee that the true entropy increases, since the optimiser can inflate the
proxy without increasing $h(\Z)$, whereas increasing a \emph{lower} bound is safe.
This bound-type asymmetry is the theoretical motivation for seeking tighter
estimators.

\paragraph{Contributions.}
We make two linked contributions.
\emph{(1) An entropy-estimator hierarchy (Section~\ref{sec:hierarchy}).} We cast
four non-contrastive regularisers (VICReg \citep{bardes2022vicreg}, LogDet
\citep{shwartzziv2023information}, PairDist \citep{kolchinsky2017estimating}, and
SIGReg \citep{balestriero2025lejepa}) as a monotone elimination of
prior-miscalibration sources. The organising result is a three-source gap
decomposition (Proposition~\ref{prop:gap}): every covariance-based entropy proxy
is loose through a non-Gaussianity gap, an off-diagonal-covariance gap, and an
estimation error, each with an explicit zero condition and the first two with a
definite sign. The hierarchy
also separates two axes, bound \emph{type} (upper/unsafe vs.\ lower/safe) and
distributional \emph{enforcement}, and shows SIGReg is the first estimator that
is simultaneously safe and exact.
\emph{(2) The SIGReg--AIF correspondence (Section~\ref{sec:method}).} We prove
that replacing VICReg with SIGReg upgrades the AIF correspondence from
approximate to exact (Theorem~\ref{thm:correspondence}): under the constant-noise
model and successful SIGReg enforcement, the gap vanishes
($\Delta_{\mathrm{SIG}}=0$), the loss admits an exact IB decomposition, and the
AIF bound $F\ge-\ln p(x)$ is preserved (the first non-contrastive regulariser
to achieve this). A dual-tightening corollary
(Corollary~\ref{cor:dual}) shows isotropy simultaneously closes the entropy gap
and the KL--MSE approximation error, making the latent goal cost an exact AIF
pragmatic-value proxy (Proposition~\ref{prop:pragmatic}). We then extend the
correspondence from the single-step free energy $F$ to the multi-step expected
free energy $G$, to ensemble-based epistemic value, and to a learned-policy
regime, and we isolate the one AIF term no current JEPA world model computes:
the \emph{state-epistemic value}, a future-state coverage signal
(Proposition~\ref{prop:o2}). Following \citet{klindt2026lejepa}, the algebraic
core of both contributions is machine-verified in the Lean~4 theorem prover,
compiling with zero \texttt{sorry} obligations (Appendix~\ref{app:lean}).

\paragraph{What the correspondence buys.}
The equivalence is not a relabelling, because it delivers what the two adjacent
results (that JEPAs already learn approximately isotropic-Gaussian embeddings
\citep{balestriero2025gaussian} and that this makes their latents identifiable and
plannable \citep{klindt2026lejepa}) do not. First, it gives a \emph{validity
criterion} for the objective: whether the surprise bound $F\ge-\ln p(x)$ survives
is a property of the estimator's bound \emph{direction}
(Proposition~\ref{prop:boundtype}), separate from how Gaussian or identifiable the
embeddings are. Second, it equips JEPA world-model planning with a \emph{principled
epistemic-value term}, the information gain that deployed planners currently
approximate only heuristically through ensemble disagreement
\citep{sobal2025closing}. Third, it turns the heuristic ``add an exploration bonus''
into a specific, signed prescription, the state-epistemic coverage term
isolated above. Fourth, it \emph{unifies} the deployed
non-contrastive JEPA world-model ecosystem with deep active inference, so each can
import the other's tools. The bare correspondence is the mechanism and these
consequences are the contribution.

\paragraph{The falsifiable hook, and a caveat.}
The distinction is sharp because it is qualitative: under SIGReg the
prior-miscalibration gap is exactly zero and the goal metric is exactly
isotropic, while under VICReg both carry an irreducible $\mathcal{O}(\delta^2)$ anisotropy
term. This produces predictions that differ by kind, not degree (e.g.,\ the
embedding condition number $\kappa\!\approx\!1$ vs.\ $\kappa\!\gg\!1$). The Gaussianity
result noted above, that \emph{any} successfully trained JEPA already learns
approximately Gaussian embeddings \citep{balestriero2025gaussian}, does not
blunt this distinction, for three reasons the theory
makes precise: ``approximately Gaussian'' is not sufficient for the AIF bound,
which is binary in the estimator's \emph{type}
(Proposition~\ref{prop:boundtype}), so VICReg may fail to preserve it even when the
residual is small; the theory governs the training \emph{path}, not only the optimum; and
the bridge is independently valuable (Section~\ref{sec:conclusion} expands all
three). A second concurrent
result \citep{klindt2026lejepa} supplies the precondition under which latent
distances are physically meaningful (Section~\ref{sec:identifiability}) and
triangulates the isotropic-Gaussian target from a third, independent direction.

\paragraph{Organisation.}
Section~\ref{sec:related} surveys related work; Section~\ref{sec:background} fixes
notation and three published facts; Section~\ref{sec:method} develops the
hierarchy, the correspondence theorem, the multi-step/EFE and learned-policy
extensions, and the testable predictions the theory entails. Empirical validation
of these predictions is left to separate work; the present paper is confined to
the theoretical development.
Section~\ref{sec:conclusion} concludes.

\section{Related Work}
\label{sec:related}

\paragraph{JEPAs and latent world models.}
Joint-embedding prediction has become the dominant SSL design for control-capable
world models \citep{lecun2022path}. Image and video variants such as I-JEPA and
V-JEPA, and the action-conditioned V-JEPA~2
\citep{assran2025vjepa2}, learn predictors in a frozen or slowly-updated latent
space; DINO-WM \citep{zhou2025dinowm} plans zero-shot on frozen DINOv2 features;
PLDM \citep{sobal2025closing} trains an end-to-end latent dynamics model with a
VICReg-style anti-collapse term and an ensemble disagreement cost; and
TD-JEPA \citep{bagatella2025tdjepa} extends latent prediction to zero-shot
reinforcement learning. \citet{terver2025drives} ablate the design space and
report that cross-entropy-method planning with an L2 terminal cost is the
strongest planner and that a short two-step training rollout is best. Most
relevant here, LeWorldModel \citep{maes2026leworldmodel} is the first JEPA world
model to train stably end-to-end from raw pixels using SIGReg as the sole
anti-collapse regulariser, with a two-term objective, no exponential-moving-average
teacher, and CEM-based model-predictive control, precisely the system our theory
formalises. Complementing these architectures, \citet{klindt2026lejepa} ask when a
JEPA trained this way earns the name ``world model'' at all, and prove that
alignment plus Gaussian regularisation recovers the world's latent degrees of
freedom up to a linear map. Their identifiability result is the precondition under
which latent distances become physically meaningful, and we return to it in
Section~\ref{sec:identifiability}.

\paragraph{Non-contrastive SSL regularisation.}
A family of regularisers prevents representational collapse without negative
pairs by shaping the embedding covariance: VICReg's variance/covariance terms
\citep{bardes2022vicreg}, Barlow Twins' redundancy reduction
\citep{zbontar2021barlow}, and Whitening \citep{ermolov2021whitening}.
\citet{shwartzziv2023information} recast VICReg's variance/covariance terms as a
Gaussian entropy estimator and place it alongside a tighter log-determinant
(LogDet) estimator. \citet{kolchinsky2017estimating} provide a non-parametric
pairwise-distance entropy bound (PairDist). SIGReg \citep{balestriero2025lejepa}
departs from moment matching entirely, enforcing an isotropic-Gaussian
distribution by random-projection Gaussianity testing and recovering VICReg as a
degenerate moment-matching special case. We unify these under a single axis:
the quality of the implicit AIF prior, i.e.\ the \emph{direction} and tightness of
the entropy bound and hence whether the AIF surprise bound survives, and supply
the information-theoretic reason to prefer SIGReg. This bound-\emph{safety}
ordering is distinct from the Gaussianity-\emph{strength} hierarchy of
\citet{klindt2026lejepa}, which ranks the same regularisers (implicit,
second-moment, full) by how strongly they drive the embeddings toward Gaussian.
Instead, ours orders by whether the variational bound is preserved, not by how tight the
distributional constraint is. Beyond self-supervised pretraining, the same
isotropic-Gaussian regularisation has independently been used to stabilise deep
reinforcement learning under non-stationarity, mitigating representation collapse
and plasticity loss \citep{pasand2026stable}. This is external evidence, in a control
setting, that the isotropic-Gaussian structure our correspondence relies on
carries concrete downstream benefits.

\paragraph{Active inference and deep generative agents.}
AIF \citep{friston2010free,smith2022tutorial} casts perception and action as free
energy minimisation. Deep instantiations scale it with amortised inference and
learned dynamics: Monte-Carlo deep active inference \citep{fountas2020deep},
the deep-learning treatment of \citet{mazzaglia2022free}, and Contrastive Active
Inference \citep{mazzaglia2021contrastive}, which replaces the intractable
decoder with an InfoNCE critic and is, to our knowledge, the only prior route
that achieves an \emph{exact} AIF correspondence in a latent model. Latent-imagination
control such as Dreamer \citep{hafner2020dream} shares the rollout-and-score
structure but is not framed as free-energy minimisation. Two concurrent \emph{variational} JEPAs recover an AIF reading by building
probabilistic structure into the architecture: Var-JEPA \citep{gogl2026varjepa}
reformulates JEPA as a coupled variational autoencoder and derives an
information-bottleneck decomposition by ELBO surgery, while the
separately-authored VJEPA of \citet{huang2026vjepa} (distinct from Meta's
V-JEPA) equips the predictor with a learned covariance and a KL-to-prior term
and proves an objective-side collapse-avoidance result. Both are, in our terms,
the \emph{easy} case, i.e., they build in the very stochasticity active inference
presupposes, and, as \citet{huang2026vjepa} makes explicit, model belief
dynamics only, omitting the sensory likelihoods, preferences, and policy
evaluation that a full active-inference agent requires. The contribution of this
paper is the \emph{hard} case: the deterministic, non-contrastive architecture
that every deployed JEPA actually instantiates, in which the probabilistic
structure must be recovered rather than assumed. It further supplies the
expected-free-energy layer that these belief-only formulations leave out, namely
pragmatic value, epistemic value, and policy-as-inference. Where \citet{mazzaglia2021contrastive} establish one route to exact latent
AIF, the contrastive one built around an InfoNCE critic, our result covers the
route the deployed JEPA world-model ecosystem already instantiates (V-JEPA~2,
DINO-WM, PLDM, LeWorldModel) and states the condition under which that objective is
a valid free energy. Two results from model-based control bear on our learned-policy regime
(Section~\ref{sec:method-policy}). \citet{georgiev2024pwm} show that a
well-regularised world model induces a \emph{smoother} optimisation landscape than
the true dynamics, making first-order policy extraction effective where
gradient-free planners (CEM, MPPI) are the default; our correspondence supplies the
regulariser-side condition they leave open, since the embedding condition number
$\kappa$ measures that conditioning and only isotropy drives it to one.
\citet{wang2023optimal} show that the optimal goal-reaching \emph{cost-to-go} is a
quasimetric, asymmetric because reachability is irreversible. This does not
contradict Proposition~\ref{prop:pragmatic}, since pragmatic value is a log goal
\emph{prior} and preferences over states are symmetric by construction, but it does
bound the regime in which the goal cost may be read as a value estimate
(Section~\ref{sec:limitations}).

\paragraph{Information bottleneck and mutual-information estimation.}
The IB principle \citep{tishby2000information} underlies both deep variational IB
\citep{alemi2017deep} and multi-view IB \citep{federici2020learning}.
\citet{poole2019variational} survey variational MI bounds, and
\citet{kolchinsky2019caveats} expose the degeneracy of MI for deterministic
encoders, the issue that the constant-noise model resolves and that motivates
the entropy-maximisation view we adopt. The same work documents two further IB
pathologies: the failure to recover the IB curve under deterministic targets and
the degenerate landscape of the IB Lagrangian. Neither bears on the present
construction, because our informativeness term is a \emph{single} entropy quantity
$h(\Z)$ rather than a $\beta$-weighted IB trade-off whose curve must be recovered;
we maximise one entropy under a distributional constraint, not a two-term
information trade-off. Relative to all of the above, our
contribution is to make the JEPA/AIF correspondence \emph{exact} for a
non-contrastive, deterministic encoder by controlling the prior-miscalibration
gap, and to trace the consequences through multi-step planning. Because the JEPA
line of work shares authorship, Appendix~\ref{app:positioning} states precisely
how this paper differs from each prior result and discusses the relevance of the
active-inference perspective it introduces.

\section{Background}
\label{sec:background}

\paragraph{Notation.}
All logarithms are natural. The latent space is $\mathcal{Z}=\R^d$ and the observation space is $\mathcal{X}$. An encoder
$f_\phi:\mathcal{X}\to\mathcal{Z}$ maps an observation $x_t$ to an embedding
$\z_t=f_\phi(x_t)$, and a predictor $P_\xi:\mathcal{Z}\times\mathcal{A}\to
\mathcal{Z}$ produces $\hat\z_t=P_\xi(\z_{t-1},a_{t-1})$. We write $q_\phi(s_t\mid
x_t)$ for the encoder/posterior, $p_\xi(s_t\mid s_{t-1},a_{t-1})$ for the
transition prior, $h(\cdot)$ for differential entropy, $I(\cdot;\cdot)$ for
mutual information, and $\KL{\cdot}{\cdot}$ for the Kullback--Leibler divergence. Throughout, $\mathcal{O}(\cdot)$ denotes the Landau (big-$O$) order symbol.

\paragraph{Fact 1 (AIF free energy).}
An AIF agent minimises the variational free energy, which decomposes as
\citep{mazzaglia2022free,smith2022tutorial}
\begin{equation}
F \;=\; \underbrace{\KL{q_\phi(s_t\mid x_t)}{p_\xi(s_t\mid s_{t-1},a_{t-1})}}
        _{\text{complexity}}
   \;-\; \underbrace{\E_{q_\phi}\!\left[\log p(x_t\mid s_t)\right]}_{\text{accuracy}}.
\label{eq:vfe}
\end{equation}
Replacing the intractable decoder by a mutual-information term and adding
$\log p(x_t)$ yields the equivalent MI-form $F^{+}=\KL{q_\phi(s_t\mid x_t)}
{p(s_t)}-I(S_t;X_t)$, an instance of the IB Lagrangian
\citep{tishby2000information,mazzaglia2021contrastive}. The informativeness term
$I(S_t;X_t)$ is generally intractable; how it is estimated determines the
agent's prior calibration.

\paragraph{Fact 2 (constant-noise encoders).}
A strictly deterministic encoder gives a degenerate $I(\Z;X)=+\infty$. The
standard remedy \citep{kolchinsky2019caveats} adds fixed, small, isotropic
observation noise, $\Z=f_\phi(X)+\epsilon$ with $\epsilon\sim\Norm{0}
{\sigma^2_{\mathrm{noise}}I_d}$, so that the conditional entropy is a constant
$\Ce=\tfrac{d}{2}\ln(2\pi e\,\sigma^2_{\mathrm{noise}})$ independent of $\phi$.
Then $I(\Z;X)=h(\Z)-\Ce$, and
\begin{equation}
\arg\max_\phi I(\Z;X) \;=\; \arg\max_\phi h(\Z).
\label{eq:io-entropy}
\end{equation}
Maximising AIF informativeness is therefore equivalent to maximising the
marginal differential entropy of the embeddings; any regulariser that maintains
$h(\Z)$ serves as the informativeness proxy. The status of
$\sigma_{\mathrm{noise}}$ as an interpretive device rather than a property of
deployed JEPAs is discussed in \S\ref{sec:limitations}.

\paragraph{Fact 3 (VICReg as an upper-bound entropy proxy).}
\citet{shwartzziv2023information} show that VICReg's variance/covariance terms
constitute the Gaussian entropy estimator $\hVR(\Z)\approx\tfrac{1}{2}\sum_k\ln
\hat\sigma^2_k$ (plus a constant). Combining Hadamard's inequality
\citep[Thm.~7.8.1]{horn2012matrix} with the maximum-entropy theorem
\citep[Thm.~8.6.5]{cover2006elements} gives the sandwich
\begin{equation}
\underbrace{\hPD(\Z)}_{\text{lower bound}}
\;\le\; h(\Z)
\;\le\; \underbrace{\tfrac{1}{2}\ln|\Sigma|+C}_{=\,\hN(\Z)\ \text{(LogDet, tight UB)}}
\;\le\; \underbrace{\tfrac{1}{2}\textstyle\sum_k\ln\sigma^2_k+C}_{\text{VICReg (loose UB)}},
\label{eq:bound-chain}
\end{equation}
where $\hN(\Z)=\tfrac{1}{2}\ln|2\pi e\,\Sigma|$ is the Gaussian reference entropy
and $\Sigma=\Cov(\Z)$. VICReg and LogDet bound $h(\Z)$ from \emph{above};
PairDist bounds it from below. Maximising an upper bound is unsafe in the sense
of Section~\ref{sec:intro}.

\paragraph{The Gaussian bridge.}
The complexity term connects to the prediction MSE used in practice through a
standard Gaussian identity. For $q_\varepsilon=\Norm{\mu_q}{\varepsilon I_d}$ and
$p=\Norm{\mu_p}{\sigma^2 I_d}$,
\begin{equation}
\KL{q_\varepsilon}{p}
= \frac{1}{2\sigma^2}\|\mu_q-\mu_p\|_2^2
  + \frac{d}{2}\!\left(\frac{\varepsilon}{\sigma^2}-1-\ln\frac{\varepsilon}{\sigma^2}\right),
\label{eq:bridge}
\end{equation}
and since the second term is independent of $\mu_p$, the optimiser sets coincide:
$\arg\min_\xi\KL{q_\varepsilon}{p_\xi}=\arg\min_\xi\|\mu_q-\mu_{p_\xi}\|_2^2$
\citep[\S10.1.1]{bishop2006pattern}. With $\mu_q=\z_t$ and
$\mu_{p_\xi}=\hat\z_t$, the JEPA prediction loss $\|\hat\z_t-\z_t\|^2$ is
optimisation-equivalent to the AIF complexity term. When variances depart from
isotropy by ratios $1+\delta_k$, the residual is $\mathcal{O}(\delta^2)$; this residual
will be eliminated exactly under isotropy (Corollary~\ref{cor:dual}).

\section{Method: The SIGReg--AIF Correspondence}
\label{sec:method}

We first define the prior-miscalibration gap and decompose it
(\S\ref{sec:hierarchy}); we then state the single-step correspondence theorem
(\S\ref{sec:thm}), extend it to multi-step planning and the expected free energy
(\S\S\ref{sec:method-efe}--\ref{sec:method-efe-decomp}) and to a learned policy (\S\ref{sec:method-policy}), and
finally derive the testable predictions (\S\ref{sec:predictions}). The full proof
development for the multi-step and expected-free-energy results is given in
Appendix~\ref{app:proofs}.

Table~\ref{tab:dictionary} fixes the correspondence at the level of objects: each
Active Inference quantity, its JEPA world-model counterpart, and the equation in
this paper where the identification is made. The table is the reading key for the
rest of the section; the single row with no JEPA counterpart, the state
epistemic value, is the paper's central structural finding
(Proposition~\ref{prop:o2}).

\begin{table}[t]
\centering
\small
\caption{The Active Inference $\leftrightarrow$ JEPA dictionary. Each row pairs an
AIF object with its JEPA world-model counterpart under the constant-noise model
and the defining relation in this paper. Informativeness and the state-epistemic
value are listed separately: the former is enforced by SIGReg, the latter has no
counterpart in any current JEPA world model.}
\label{tab:dictionary}
\setlength{\tabcolsep}{5pt}
\renewcommand{\arraystretch}{1.25}
\begin{tabular}{@{}lll@{}}
\toprule
AIF object / role & JEPA counterpart & Defining relation \\
\midrule
Hidden state $s_t$ & Embedding $\z_t=f_\phi(x_t)$ & encoder; Fact~1 \\
Observation $x_t$ & Input observation & Raw sensory data $x_t \in \mathcal{X}$ \\
Action $a_t$ & Predictor action arg. ($P_\xi:\Z\times\mathcal{A}\!\to\!\Z$) & Control command $a_t\in\mathcal{A}$ \\
Policy $\pi$ & Action sequence / learned policy & \S\ref{sec:method-policy} \\
Transition prior $p(\z_\tau\mid \z_{\tau-1},a_{\tau-1})$ & Gaussian dynamics about $P_\xi$ & Eq.~\eqref{eq:genmodel} \\
Posterior $q_\phi(\z_t\mid x_t)$ & Encoder $+$ constant noise & Fact~2 \\
Complexity $\KL{q_\phi}{p_\xi}$ & Prediction MSE (Gaussian bridge) & Eq.~\eqref{eq:bridge} \\
Informativeness $I(Z;X)$ & SIGReg-enforced $\Istar$ & Def.~\ref{def:sigreg}, Prop.~\ref{prop:sigzero} \\
Pragmatic value & Latent goal cost $\|\z_\tau-\z_g\|^2$ & Prop.~\ref{prop:pragmatic} \\
Param.\ info gain & Ensemble predictive variance & \S\ref{sec:method-efe-decomp} \\
\textbf{State epistemic value} $h(Z_\tau\mid\pi)-\Ce$ & \textbf{absent in JEPA} & Prop.~\ref{prop:o2} \\
\bottomrule
\end{tabular}
\end{table}

\subsection{The prior-miscalibration gap and the entropy-estimator hierarchy}
\label{sec:hierarchy}

\begin{definition}[Prior-miscalibration gap]
\label{def:gap}
For an entropy proxy $\hat H(\Z)$ used in place of $h(\Z)$ in the AIF
informativeness term, the prior-miscalibration gap is
$\Delta_{\hat H}(\Z)\coloneqq h(\Z)-\hat H(\Z)$.
A lower-bound proxy has $\Delta\ge0$ (it underestimates achievable
informativeness); an upper-bound proxy has $\Delta\le0$ (it overpromises, and
maximising it need not increase $h(\Z)$).
\end{definition}

\begin{proposition}[Three-source gap decomposition]
\label{prop:gap}
For any embedding distribution $p_\Z$ with finite covariance $\Sigma\succ0$ and
Gaussian reference entropy $\hN(\Z)=\tfrac12\ln|2\pi e\,\Sigma|$,
\begin{equation}
\Delta_{\hat H}(\Z)
= \underbrace{\bigl[h(\Z)-\hN(\Z)\bigr]}_{\text{(I) non-Gaussianity}\ \le\,0}
+ \underbrace{\bigl[\hN(\Z)-\hat H(\Z)\bigr]}_{\text{(II)+(III) estimation}} .
\label{eq:gap-decomp}
\end{equation}
Gap I is $\le 0$ with equality iff $p_\Z$ is Gaussian (maximum-entropy theorem,
\citealp[Thm.~8.6.5]{cover2006elements}). For a variance-based proxy $\hat
H(\Z)=\tfrac12\sum_k\ln\sigma^2_k+\tfrac d2\ln2\pi e$, the estimation gap
contains an off-diagonal term $\tfrac12\ln|\Sigma|-\tfrac12\sum_k\ln\sigma^2_k
\le 0$ (Gap II), which is $0$ iff $\Sigma$ is diagonal (Hadamard's inequality,
\citealp[Thm.~7.8.1]{horn2012matrix}); the remaining term (Gap III) is the
estimator's own error and is $0$ iff $\hat H=\hN$ exactly.
\end{proposition}

\paragraph{Scope of the decomposition, and two ways a gap can be absent.}
The decomposition routes the total gap through the Gaussian reference $\hN$ and so
characterises \emph{Gaussian-reference} (covariance-based) proxies, namely VICReg and
LogDet. An estimator that invokes no Gaussian reference does not incur Gaps I--II at
all: PairDist \citep{kolchinsky2017estimating} bounds the mixture entropy directly
from pairwise distances between components, without moment-matching $p_\Z$ to a
single Gaussian and without diagonalising $\Sigma$, so its entire slack is Gap III.
This is distinct from the gap's zero \emph{condition} holding. Gaps I and II vanish
iff $p_\Z$ is Gaussian and $\Sigma$ is diagonal respectively, which are properties of
the \emph{distribution}, not of the estimator; no estimator makes them true merely by
declining to use them. Only SIGReg drives $p_\Z$ to satisfy them. Accordingly,
Table~\ref{tab:hierarchy} marks a gap \emph{not incurred} when the estimator never
invokes that source, and \emph{enforced} when the estimator makes the zero condition
hold. The hierarchy is monotone in the first sense (two sources incurred, then one,
then none) and SIGReg alone supplies the second.

\begin{proof}[Proof sketch]
Insert $\pm\hN(\Z)$ for \eqref{eq:gap-decomp}. Gap I follows from the
maximum-entropy theorem; among distributions with covariance $\Sigma$ the
Gaussian uniquely maximises differential entropy. Gap II follows from Hadamard's
inequality $|\Sigma|\le\prod_k\sigma^2_k$ with equality iff $\Sigma$ is diagonal.
Gap III is the residual by construction.
\end{proof}

\begin{proposition}[Bound type and AIF-bound preservation]
\label{prop:boundtype}
Let $\hat F^{+}$ be the proxy free energy obtained by substituting $\hat H(\Z)$
for $h(\Z)$. Then $\hat F^{+}-F^{+}=\Delta_{\hat H}(\Z)$, so:
\emph{(i)} a lower-bound proxy ($\hat H\le h$, e.g.\ PairDist) gives $\hat
F^{+}\ge F^{+}\ge-\ln p(x)$, so the AIF bound is \emph{preserved by construction};
\emph{(ii)} an upper-bound proxy ($\hat H\ge h$, e.g.\ VICReg or LogDet) gives
$\hat F^{+}\le F^{+}$, which no longer certifies $\hat F^{+}\ge-\ln p(x)$: the AIF
bound is \emph{not guaranteed to be preserved}. The inequality may still hold for a
particular $p_\Z$; what is lost is the guarantee, and with it the licence to minimise
$\hat F^{+}$ as a surrogate for surprise. Maximising an upper bound is therefore
\emph{unsafe} rather than necessarily wrong: the slack $\hat H-h\ge0$ can grow under
optimisation (Remark~\ref{rem:hypercube}).
\end{proposition}

\begin{remark}[Why an upper bound is unsafe: a worked case]
\label{rem:hypercube}
The failure in Proposition~\ref{prop:boundtype}(ii) is \emph{possible}, not
inevitable, and that possibility is precisely what makes an upper bound unsafe to
maximise. Because $\hat H-h\ge0$, an optimiser can raise $\hat H$ by inflating the
slack instead of the entropy. Let $p_\Z$ be uniform on an axis-aligned hypercube with
per-coordinate variance $\sigma^2$. Its off-diagonal covariances vanish and its
marginal variances are $\sigma^2$, so the VICReg proxy
$\hVR(\Z)=\tfrac12\sum_k\ln\sigma^2_k+C$ assigns it exactly the value it assigns an
isotropic Gaussian of the same variance. The true entropies differ, however:
$h(\Z)=d\ln(\sqrt{12}\,\sigma)$ for the hypercube against
$\hN(\Z)=\tfrac d2\ln(2\pi e\sigma^2)$ for the Gaussian, and $2\pi e>12$, so the
hypercube is strictly sub-maximal by the full non-Gaussianity gap
\citep[Thm.~3]{balestriero2025lejepa}. Maximising $\hVR$ therefore exerts no pressure
toward the entropy maximiser: the proxy can sit at its maximum while $h(\Z)$ does not.
Nothing here forces $\hat F^{+}<-\ln p(x)$ for any given input; what is lost is the
certificate. A lower bound admits no such manoeuvre, since raising $\hat H\le h$ can
only raise $h$.
\end{remark}

\paragraph{The hierarchy.}
Proposition~\ref{prop:gap} indexes four regularisers by which gap sources they
\emph{incur}, and Proposition~\ref{prop:boundtype} adds the orthogonal bound-type
axis. Table~\ref{tab:hierarchy} summarises the resulting hierarchy
(Figures~\ref{app:fig:hierarchy} and~\ref{app:fig:design} render the bound-type
sandwich and the two-dimensional design space). VICReg (diagonal Gaussian proxy)
incurs Gaps I and II and is an unsafe upper bound; LogDet (full-covariance proxy)
does not incur Gap II, but still incurs Gap I and remains an unsafe upper bound;
PairDist (non-parametric) incurs neither, and is a safe lower bound, but it enforces
no distributional shape; and SIGReg enforces the isotropic Gaussian directly, making
all three zero conditions hold. Safety and enforcement are therefore independent
properties, and only SIGReg has both.

\begin{table}[t]
\centering
\small
\caption{The entropy-estimator hierarchy as progressive elimination of
prior-miscalibration sources. ``UB''/``LB'' denote upper/lower bounds on the latent
entropy. The gap entries distinguish two different ways a gap can be absent, a
distinction the estimators do not share. \emph{Not incurred}: the estimator's
construction never invokes that source, so no looseness arises from it. LogDet
computes $\ln|\Sigma|$ and so never makes the diagonal approximation; PairDist bounds
the mixture entropy directly from pairwise component distances and so never invokes a
Gaussian reference at all. \emph{Enforced}: the estimator actively drives $p_\Z$ to
satisfy the gap's zero condition (Gaussianity, diagonality). Only SIGReg does the
latter. Being \emph{safe} is therefore not the same as \emph{enforcing}: PairDist is
safe (lower bound) but does not shape $p_\Z$; SIGReg is both. Figure~\ref{app:fig:design}
renders the same distinction as the ``Gaussian not assumed / assumed / enforced'' axis.}
\label{tab:hierarchy}
\setlength{\tabcolsep}{4pt}
\begin{tabular}{llcccc}
\toprule
Level & Proxy & Bound type & Gap I & Gap II & Gap III \\
& & & (non-Gauss.) & (off-diag.) & (estim.) \\
\midrule
0 & VICReg \citep{bardes2022vicreg}          & UB (loose)     & incurred      & incurred      & ---         \\
1 & LogDet \citep{shwartzziv2023information} & UB (tight)     & incurred      & not incurred  & small       \\
2 & PairDist \citep{kolchinsky2017estimating}& \textbf{LB}    & not incurred  & not incurred  & finite-$N$  \\
3 & \textbf{SIGReg} \citep{balestriero2025lejepa} & enforce.\ & \textbf{enforced} & \textbf{enforced} & \textbf{enforced} \\
\bottomrule
\end{tabular}
\end{table}

\begin{definition}[SIGReg; \citealp{balestriero2025lejepa}]
\label{def:sigreg}
Sketched Isotropic Gaussian Regularisation enforces $p_\Z\approx\Norm{0}{I_d}$ by
random-projection Gaussianity testing: for unit directions
$\mathbb{A}\subset\mathcal{S}^{d-1}$ and a univariate normality statistic $T$
(the Epps--Pulley test, \citealp{epps1983test}),
$\Sig_T(\mathbb{A},\{\z_n\})=\tfrac{1}{|\mathbb{A}|}\sum_{a\in\mathbb{A}}
T(\{a^\top\z_n\})$. The hyperspherical Cram\'er--Wold theorem
\citep{cramer1936some} guarantees that matching all one-dimensional projections
to a Gaussian is sufficient for a full distributional match.
\end{definition}

\begin{proposition}[SIGReg eliminates the gap]
\label{prop:sigzero}
Under Fact~2, if SIGReg enforces $p_\Z=\Norm{0}{\tfrac{c}{d}I_d}$, where
$c\coloneqq\tr(\Sigma_\Z)$ is the target total variance, so the per-coordinate
variance is $c/d$, then
$h(\Z)=\hstar(c,d)\coloneqq\tfrac d2\ln(2\pi e\,c/d)$, the gap is exactly zero
($\Delta=0$ with Gaps I, II, III all zero), and the informativeness term attains
its constrained maximum $I(\Z;X)=\hstar(c,d)-\Ce\eqqcolon\Istar$. Gaussianity
alone closes the gap ($\Delta=0$); isotropy additionally maximises the value of
the closed bound.
\end{proposition}

\subsection{The single-step correspondence theorem}
\label{sec:thm}

\begin{definition}[SIGReg-modular free energy]
\label{def:sigfe}
Under the constant-noise model with deterministic $\z_t,\hat\z_t$, the
SIGReg-modular free energy is
$\Fsig=\tfrac{1}{2\sigma^2}\|\hat\z_t-\z_t\|_2^2+C_{\mathrm{KL}}
-\lambda\,\Istar\,[1-\Sig_T(\mathbb{A},\{\z_n\})]$,
where $C_{\mathrm{KL}}$ is the $\mu_p$-independent second term of the Gaussian
bridge \eqref{eq:bridge} and $\Istar=\hstar(c,d)-\Ce$
(Proposition~\ref{prop:sigzero}); both are constants in $(\phi,\xi)$. This is
the AIF reading of the LeJEPA objective $\mathcal{L}_{\mathrm{LeJEPA}}=
\|\hat\z_t-\z_t\|^2+\lambda_{\mathrm{SIG}}\,\Sig_T(\mathbb{A},\{\z_n\})$
\citep{balestriero2025lejepa}, up to $(\phi,\xi)$-independent additive
constants and the rescaling $\lambda_{\mathrm{SIG}}=2\sigma^2\lambda\,\Istar$
(with $\Istar>0$, i.e.\ $c/d>\sigma^2_{\mathrm{noise}}$, which holds for the
small fixed noise of Fact~2):
once enforcement succeeds, the informativeness is not estimated but
\emph{guaranteed} to equal the constant $\Istar$
(Figure~\ref{app:fig:architecture} renders this reading of the objective).
\end{definition}

\begin{theorem}[SIGReg--AIF correspondence]
\label{thm:correspondence}
Under the constant-noise model (Fact~2), the Gaussian encoder family, and
successful SIGReg enforcement ($p_\Z=\Norm{0}{\tfrac cd I_d}$ in the population
limit $M,N\to\infty$), taking $\lambda=1$ in Definition~\ref{def:sigfe}:
\begin{enumerate}[label=(\roman*),leftmargin=2.2em,itemsep=1pt,topsep=2pt]
\item \textbf{Zero gap.} $\Delta_{\mathrm{SIG}}\coloneqq h(\Z)-\hstar(c,d)=0$;
Gaps I, II, III all vanish.
\item \textbf{Exact IB form (within the constant-noise model).}
$\Fsig\big|_{\Sig=0}=\tfrac{1}{2\sigma^2}\|\hat\z_t-\z_t\|^2+C_{\mathrm{KL}}
-\Istar=F^{+}$, where the complexity term
$\tfrac{1}{2\sigma^2}\|\hat\z_t-\z_t\|^2+C_{\mathrm{KL}}$ equals
$\KL{q_\phi}{p_\xi}$ \emph{exactly} (not merely in gradient) because the
Gaussian bridge \eqref{eq:bridge} has zero anisotropy error under isotropy
(Corollary~\ref{cor:dual}), and the informativeness term equals
$I(\Z;X)=\Istar$ exactly because SIGReg enforces the maximum-entropy
distribution.
\item \textbf{AIF bound preserved.} $\Fsig=F^{+}\ge-\ln p(x)$. By contrast
$\Fvr\le F^{+}$ (Proposition~\ref{prop:boundtype}(ii)), so under VICReg the bound is
no longer guaranteed.
\item \textbf{Contrastive equivalence.} $\Fsig$ attains the same guarantees as
the contrastive free energy $F_{\mathrm{NCE}}$ \citep{mazzaglia2021contrastive}
through the non-contrastive pathway.
\end{enumerate}
\end{theorem}

\begin{proof}[Proof sketch]
From the MI-form $F^{+}=\KL{q_\phi}{p_\xi}-h(\Z)+\Ce$, enforcement gives
$h(\Z)=\hstar(c,d)$ (Proposition~\ref{prop:sigzero}), hence (i); isotropy makes
the Gaussian-bridge \eqref{eq:bridge} exact, and substituting it together with
$I(\Z;X)=\Istar$ into the MI-form gives (ii) with all additive constants
carried explicitly; (iii) follows from
Proposition~\ref{prop:boundtype} at $\Delta=0$; (iv) by comparison with the
InfoNCE optimum. The multi-step extension and full derivations are in
Appendix~\ref{app:proofs}; Table~\ref{tab:status} (\S\ref{sec:limitations})
records which parts of this and the other main results are exact under the
main-text assumptions and which are bounded or deferred.
\end{proof}

\begin{remark}[Which assumptions do the work]
\label{rem:scope}
Theorem~\ref{thm:correspondence} is exact \emph{within} its hypotheses, and it is
worth naming what each one does. The \emph{constant-noise model} (Fact~2) is what
makes $I(\Z;X)$ well defined at all for a deterministic encoder, so it is presupposed
by every information-theoretic statement here; it is assumed identically for VICReg
and for SIGReg, and therefore does no work in the comparison between them.
\emph{Successful SIGReg enforcement} does the real work: it supplies
$p_\Z=\Norm{0}{\tfrac cd I_d}$ and with it, simultaneously, the zero entropy gap~(i),
the exactness of the Gaussian bridge~(ii) via Corollary~\ref{cor:dual}, and the bound
preservation~(iii). The \emph{population limit} $M,N\to\infty$ is what makes
enforcement exact rather than approximate; at finite $(M,N)$ the guarantee degrades
gracefully and linearly in the enforcement residual (Corollary~\ref{cor:graceful}), at
the rate of Proposition~\ref{prop:rate}, rather than failing discontinuously. Away
from these conditions the correspondence is approximate, not void:
Table~\ref{tab:status} records the status result by result.
\end{remark}

\begin{corollary}[Dual tightening]
\label{cor:dual}
Isotropy ($\Sigma=\tfrac cd I_d$) closes the entropy gap (Gap II) and the
Gaussian-bridge approximation error \emph{simultaneously}, because both are
governed by the same anisotropy quantity: the off-diagonal/eigenvalue spread of
$\Sigma$. Hence SIGReg tightens the informativeness and complexity terms of the
free energy with a single distributional constraint; under VICReg both carry an
$\mathcal{O}(\delta^2)$ residual.
\end{corollary}

\begin{corollary}[Graceful degradation under approximate enforcement]
\label{cor:graceful}
Suppose SIGReg enforcement is only approximate, leaving a residual gap
$\Delta_{\mathrm{SIG}}=\epsilon\ge0$ and embedding covariance $\Sigma$ with
anisotropy $\|\Sigma-\tfrac cd I_d\|$. Then \emph{(i)} by
Proposition~\ref{prop:boundtype} the AIF-bound slack is exactly
$|\hat F^{+}-F^{+}|=\epsilon$, linear in the residual rather than catastrophic;
and \emph{(ii)} the pragmatic-value distortion is first order in the anisotropy
$\|\Sigma-\tfrac cd I_d\|$, the same quantity that governs both gaps by
Corollary~\ref{cor:dual}, recovering the VICReg $\mathcal{O}(\delta^2)$ expression as the
special case in which the eigenvalue spread equals $\delta$. The qualitative
distinction is the load-bearing point: under SIGReg $\epsilon$ is
\emph{penalised directly by the objective and driven toward zero along the
optimisation path} (self-correcting), whereas under VICReg the anisotropy
$\delta$ is \emph{structural and irreducible}, a fixed floor present even at the
optimum.
\end{corollary}

\begin{proof}[Proof sketch]
Part (i) is immediate from Proposition~\ref{prop:boundtype}: the proxy free energy
differs from $F^{+}$ by exactly the substituted gap, here $\epsilon$. For (ii),
expand the Gaussian-bridge and informativeness terms about the isotropic point;
Corollary~\ref{cor:dual} identifies their common first-order coefficient as the
anisotropy $\|\Sigma-\tfrac cd I_d\|$, and the VICReg case is the diagonal-but-%
unequal-variance specialisation with spread $\delta$. The self-correction claim
is that SIGReg's objective contains $\Sig_T$, which penalises $\epsilon$ directly,
whereas VICReg's variance-covariance penalty admits an anisotropic optimum. The
constant in (ii) inherits the finite-$d$ caveat of Remark~\ref{rem:rate}.
\end{proof}

\begin{proposition}[Finite-sample rate]
\label{prop:rate}
With $M$ projections, batch size $N$, and projected-density Sobolev regularity
$\alpha$ on $\mathcal{S}^{d-1}$,
$\lvert h(\Z)-\hstar(c,d)\rvert
= \mathcal{O}\!\bigl(M^{-2\alpha/(d-1)}\bigr)+\mathcal{O}\!\bigl(N^{-1/2}\bigr)$,
combining the directional discrepancy rate of \citet[Thm.~5]{balestriero2025lejepa},
a Cram\'er--Wold conversion \citep{cramer1936some}, an entropy--total-variation
continuity bound \citep{bobkov2011concentration}, and the $\sqrt N$ U-statistic
rate \citep{hoeffding1948class,serfling1980approximation}. Because the
isotropic-Gaussian target is $C^\infty$, $\alpha$ grows during training and the
$M$-dependent term accelerates.
\end{proposition}

\begin{remark}[Honest status of the rate]
\label{rem:rate}
Three links in Proposition~\ref{prop:rate} invoke published results directly and
are rigorous; the quantitative Cram\'er--Wold step, converting uniform
directional convergence to a multivariate total-variation bound for the
non-independent random directions SIGReg uses, is classical only in its
qualitative form. A fully explicit finite-$d$ constant requires a quantitative
multivariate Berry--Esseen / Stein argument \citep{bonis2020stein,vershynin2018high}
and may be looser than the displayed rate suggests. This does not affect the
qualitative conclusion that the gap vanishes as $M,N\to\infty$; we flag it as a
point for a camera-ready strengthening.
\end{remark}

\subsection{From \texorpdfstring{$F$}{F} to \texorpdfstring{$G$}{G}: multi-step
planning and the expected free energy}
\label{sec:method-efe}

JEPA world models exist to plan: at deployment the predictor is unrolled $H$
steps and scored against a goal, which in AIF corresponds to minimising the
expected free energy (EFE) $G_\pi$ over policies. $G_\pi$ adds two elements
absent from the single-step $F$: a multi-step temporal structure in which errors
compound, and an epistemic-value term that drives exploration. Two questions
follow: does the per-step KL--MSE exactness compose across the horizon, and does
each EFE term map to a JEPA planning-cost component?

\paragraph{Multi-step complexity.}
Three deployment regimes bound how the per-step exactness composes. Under
\emph{teacher-forcing}, where each prediction is grounded in the true encoded
observation, the per-step Gaussian-bridge exactness composes exactly. Under
\emph{free autoregressive rollout}, where the predictor is unrolled on its own
output, two errors enter: a small-noise Jacobian term scaling with the predictor
sensitivity, and a compounding term governed by the predictor's Lipschitz
constant $L_P$ (the factor by which a latent perturbation grows per predicted
step). Under \emph{model-predictive control} with replanning interval $m$, the
rollout only scores candidate actions and the agent re-encodes a true observation
every $m$ steps, so the effective autoregressive horizon is $m$, not the full
planning horizon $H$. The following theorem collects the three regimes.

\begin{theorem}[SIGReg--AIF multi-step correspondence]
\label{thm:multistep}
Let the encoder and predictor satisfy the constant-noise model and successful
SIGReg enforcement (the hypotheses of Theorem~\ref{thm:correspondence}), and let
the predictor be $L_P$-Lipschitz in its latent argument. Under
model-predictive control with planning horizon $H$ and replanning interval $m$:
\begin{enumerate}[label=(\roman*),leftmargin=2.2em,itemsep=1pt,topsep=2pt]
\item \textbf{Teacher-forced exactness.} The teacher-forced multi-step prediction
cost is an exact proxy for the joint trajectory complexity $\KL{q}{p}$, and the
AIF bound is preserved at the trajectory level.
\item \textbf{Autoregressive control.} The autoregressive planning cost
approximates the trajectory complexity with total error bounded by a Jacobian
term of order $H\varepsilon^2/\sigma^2$ plus a compounding term of order
$L_P^{2H}\bar\epsilon^2/\sigma^2$, where $\bar\epsilon$ is the mean per-step
prediction error.
\item \textbf{Executed-trajectory exactness under deployment.} The same bound
holds with $m$ in place of $H$; in particular at $m=1$ (the PLDM/LeWorldModel
default) both error terms vanish and the executed trajectory satisfies the exact
per-step correspondence of Theorem~\ref{thm:correspondence}.
\item \textbf{Ranking fidelity.} The planning cost orders candidate action
sequences faithfully, with zero per-step ranking distortion.
\end{enumerate}
Replacing SIGReg with VICReg adds an $\mathcal{O}(H\delta^2)$ anisotropy term to every
bound, including a nonzero per-step ranking distortion. The full development is
given in Appendix~\ref{app:proofs} (Theorem~\ref{app:thm:multistep}).
\end{theorem}

The practical reading of part~(iii) is the one that matters for deployed systems.
The compounding term $L_P^{2H}$ in part~(ii) is vacuous only for an expansive
predictor over a long free horizon; two standard facts remove that worry. First,
under one-step-replanning MPC the effective horizon is $m=1$, so by part~(iii)
both error terms vanish on the executed trajectory regardless of $L_P$. Second,
spectral normalisation of the predictor bounds $L_P\le1$ directly, making the
compounding term non-expansive even under free rollout. Thus, under standard
deployment, the regulariser choice factors out of the horizon-dependent error
budget: SIGReg yields an exact multi-step correspondence with zero compounding
error, while VICReg's anisotropy distortion persists at every step.

\subsection{The expected free energy of a JEPA world model}
\label{sec:method-efe-decomp}

The multi-step result above concerns the \emph{complexity} term across a horizon.
Planning, however, scores trajectories against a goal, which in AIF is the
expected free energy (EFE) $G_\pi$. We now instantiate the EFE under the JEPA
generative model and read off its terms, which is where the correspondence
becomes most legible. A JEPA world model with encoder $f_\phi$, predictor
(ensemble) $P_\xi$, and dynamics noise $\sigma^2$ \emph{is} an AIF generative
model
\begin{equation}
p(x_{1:H},\z_{1:H}\mid \z_0,\pi,\xi)=\textstyle\prod_\tau p(x_\tau\mid \z_\tau)\,
p(\z_\tau\mid \z_{\tau-1},a_{\tau-1},\xi),
\label{eq:genmodel}
\end{equation}
with Gaussian dynamics $p(\z_\tau\mid \z_{\tau-1},a_{\tau-1})=
\Norm{P_\xi(\z_{\tau-1},a_{\tau-1})}{\sigma^2 I_d}$, the constant-noise observation
model of Fact~2, and Gaussian goal preferences $p(x\mid C)\propto
\exp(-\tfrac{1}{2\sigma_g^2}\|f_\phi(x)-\z_g\|^2)$. Under the mean-field posterior
$q(\z_{1:H})=\prod_\tau\Norm{f_\phi(x_\tau)}{\varepsilon^2 I_d}$ the per-step EFE
decomposes equivalently as epistemic-plus-pragmatic value or as ambiguity-plus-risk
\citep{smith2022tutorial,mazzaglia2022free}. We take each term in turn; the
expectation calculations are routine and are deferred to Appendix~\ref{app:efe},
with the term-by-term map summarised in Figure~\ref{app:fig:efe}.

\paragraph{Pragmatic value.}
The headline term is the goal cost.

\begin{proposition}[Pragmatic value under SIGReg]
\label{prop:pragmatic}
With Gaussian goal preferences $p(x\mid C)\propto\exp(-\tfrac{1}{2\sigma_g^2}
\|f_\phi(x)-\z_g\|^2)$, the pragmatic value at step $\tau$ is
$-\E_{q(x_\tau\mid\pi)}[\ln p(x_\tau\mid C)]
=\tfrac{1}{2\sigma_g^2}\|\z_\tau-\z_g\|^2+\mathrm{const}$. Under SIGReg the
Euclidean goal cost is an \emph{exact} KL proxy,
\begin{equation}
\KL{\Norm{\z_\tau}{\varepsilon^2 I_d}}{\Norm{\z_g}{\varepsilon^2 I_d}}
= \frac{1}{2\varepsilon^2}\|\z_\tau-\z_g\|^2 \quad(\text{zero anisotropy error}),
\label{eq:pragmatic-exact}
\end{equation}
whereas under VICReg it carries an $\mathcal{O}(\delta^2)$ distortion from dimension-dependent
precision weighting. PLDM's goal cost and DINO-WM's terminal cost are instances
of this term. This reading presumes the latent goal distance faithfully tracks
the world's goal distance; Section~\ref{sec:identifiability} states the
identifiability precondition under which that holds.
\end{proposition}

\paragraph{The remaining terms.}
The \emph{ambiguity} is the constant $\Ce$, independent of policy and regulariser.
The \emph{risk} reduces to the same MSE-based functional as the pragmatic value,
exact under SIGReg and anisotropic under VICReg. The \emph{parameter information
gain} maps to ensemble predictive variance, $I(\xi;\Z_\tau\mid\z_{\tau-1},
a_{\tau-1})\approx\tfrac{1}{2\sigma^2}\sum_j\Var_k[P^k_\xi(\cdot)_j]$, which is
exactly PLDM's uncertainty cost. The one term with no JEPA counterpart is the
\emph{state} epistemic value.

\begin{proposition}[The state-epistemic gap]
\label{prop:o2}
Under the constant-noise model the state-epistemic value is
$I_q(\Z_\tau;X_\tau\mid\pi)=h(\Z_\tau\mid\pi)-\Ce$, a coverage signal driving the
agent toward policies that maximise future-state entropy. Under SIGReg the
marginal satisfies $h(\Z_\tau)=\hstar(c,d)$, so $\Istar$ is a principled upper
bound; for a specific policy $h(\Z_\tau\mid\pi)\le\hstar(c,d)$. \textbf{No current
JEPA world model computes this quantity}; it is the primary structural gap
between AIF and JEPA planning.
\end{proposition}

These two informativeness-related quantities must not be conflated. The
\emph{single-step informativeness} $I(Z;X)=\Istar$ is what SIGReg \emph{enforces}
(Proposition~\ref{prop:sigzero}): it pins the marginal entropy of the embeddings
to its maximum and is fully accounted for in the correspondence. The
\emph{multi-step state-epistemic value} $h(\Z_\tau\mid\pi)-\Ce$
(Proposition~\ref{prop:o2}) is a different object, a policy-conditioned coverage
signal over future states, and it is precisely the term that no JEPA world model
computes. SIGReg guaranteeing the former says nothing about a JEPA agent
possessing the latter.

\paragraph{The consolidated objective.}
Substituting the four terms into the per-step EFE collects the entire JEPA
planning cost into a single expression,
\begin{equation}
G^{\mathrm{full}}_\pi
= \underbrace{\frac{1}{2\sigma_g^2}\sum_\tau\|\z_\tau-\z_g\|^2}_{\propto\,\Cgoal\ \text{(goal cost)}}
\;-\; \underbrace{\sum_\tau\big(h(\Z_\tau\mid\pi)-\Ce\big)}_{\text{state epistemic (absent in JEPA)}}
\;-\; \underbrace{\frac{1}{2\sigma^2}\sum_\tau\sum_j\Var_k[P^k_j]}_{\propto\,\Cunc\ \text{(ensemble var.)}}
\;+\;\mathrm{const}.
\label{eq:efe-consolidated}
\end{equation}
Equation~\eqref{eq:efe-consolidated} is the paper's central reading of JEPA
planning: the deployed planning cost $\Cgoal+\beta\,\Cunc$ is exactly the expected
free energy $G^{\mathrm{full}}_\pi$ \emph{minus} the one term no JEPA world model
computes: the state-epistemic coverage signal. Under SIGReg each retained term
is exact (Theorem~\ref{thm:correspondence}, Proposition~\ref{prop:pragmatic});
under VICReg the pragmatic, risk, and ensemble terms each carry the $\mathcal{O}(\delta^2)$
anisotropy distortion. The absent term is not a defect of the correspondence but
its sharpest empirical prediction: it names precisely what a JEPA agent would have
to add to become a complete active-inference agent
(Proposition~\ref{prop:o2}).

\subsection{The learned-policy regime}
\label{sec:method-policy}

Sections~\ref{sec:method-efe} and~\ref{sec:method-efe-decomp} address what the agent
\emph{evaluates}; AIF also
specifies how it \emph{acts}: action selection is inference, with a posterior
$q(\pi)\propto\exp(-\zeta G_\pi)$ \citep{smith2022tutorial}, which CEM/MPPI only
approximate. A learned goal-conditioned policy makes this exact.

\begin{proposition}[Amortised EFE policy]
\label{prop:policy}
Let $\pi_\phi(a_t\mid\z_t,\z_g)$ minimise $\mathcal{L}_\pi=\E_{a\sim\pi_\phi}
[\sum_\tau\gamma^\tau G^{\mathrm{SIG},(\tau)}_\pi]-\tfrac1\zeta H[\pi_\phi]$. Then
$\mathcal{L}_\pi=\tfrac1\zeta\KL{\pi_\phi}{\pi^{*}}-\tfrac1\zeta\ln\mathcal{Z}_H$
with Boltzmann optimum $\pi^{*}\propto\exp(-\zeta G^{\mathrm{SIG}})$; the policy
trains against the \emph{exact} EFE under SIGReg (Theorem~\ref{thm:correspondence},
Proposition~\ref{prop:pragmatic}) and against an $\mathcal{O}(\delta^2)$-distorted EFE under
VICReg. The precision $\zeta$ controls action entropy and maps to the AIF policy
precision; CEM is the $\zeta\!\to\!\infty$ limit and MPPI with temperature
$\lambda$ is $\zeta=1/\lambda$, so the learned policy generalises both.
\end{proposition}

The policy network supplies an explicit AIF habit prior and computable action
entropy that MPC lacks, and acts in a single forward pass rather than thousands
of model evaluations per step. A policy \emph{ensemble} partially proxies the
state-epistemic value of Proposition~\ref{prop:o2}. MPC and the learned policy are
complementary cases of the same correspondence: MPC optimises the planning cost
and so realises the AIF policy posterior of Proposition~\ref{prop:policy} at the
zero-temperature limit $\zeta\!\to\!\infty$, giving \emph{pointwise} exactness
along the single executed trajectory (Appendix~\ref{app:proofs},
Corollary~\ref{app:cor:mpc}), whereas a learned stochastic policy realises the
finite-temperature posterior itself, giving \emph{distributional} coverage over
behaviours. Under SIGReg both are exact, the per-step proxy under MPC and the
training objective under the learned policy, while VICReg distorts each by
$\mathcal{O}(\delta^2)$. The learned-policy direction carries implementation-timeline risk,
not architectural risk: the correspondence is exact at the optimum, and the open
question is amortisation capacity and latency, not whether the objective is
correct. Evidence from model-based control bounds that risk: \citet{georgiev2024pwm}
extract policies from pre-trained world models by first-order gradients in minutes
per task, outperforming planners with ground-truth dynamics, because a
well-regularised model induces a smoother optimisation landscape than the true
dynamics. Our correspondence names the regulariser-side condition for that
smoothness: gradients of $\mathcal{L}_\pi$ propagate through the latent geometry,
whose conditioning is measured by $\kappa$, so isotropy leaves the amortised
objective well-conditioned while VICReg's anisotropy stretches it along the
dominant covariance directions. This is \emph{structural} rather than a theorem of
the present paper, and it sharpens TP1: first-order amortisation should be better
conditioned under SIGReg than under VICReg.

\subsection{The precondition: linear identifiability of the latent}
\label{sec:identifiability}

The pragmatic-value reading (Proposition~\ref{prop:pragmatic}) and the planning
correspondence presume that Euclidean distance in the embedding is a faithful
surrogate for distance in the world's latent state, so far grounded only on the
constant-noise model and SIGReg isotropy. A recent identifiability result
supplies the missing recovery guarantee. \citet{klindt2026lejepa} consider a
Gaussian world (independent latents observed through an unknown nonlinear mixing,
positive pairs from an Ornstein--Uhlenbeck transition) and prove that any encoder
satisfying the two LeJEPA objectives is \emph{linearly identifiable},
$h(z)=Qz$ for orthogonal $Q$, recovering the latent up to a rotation; their
planning theorem shows that under $h(z)=Qz$ any finite-horizon problem with
$\mathcal{O}(n)$-invariant costs has identical value and optimal plan in the learned and
true latent. Squared Euclidean goal distance is exactly such a cost, so the
latent goal distance is a monotone image of the true-state goal distance by
construction: the pragmatic-value proxy acquires an external, formally verified
precondition the present theory previously assumed.

Two qualifications sharpen the position. First, identifiability is
\emph{regulariser-agnostic} at the optimum: SIGReg, VICReg, and InfoNCE all
attain near-perfect linear recovery on Gaussian-latent data
\citep{klindt2026lejepa}, so it is not the axis on which to prefer SIGReg; the
discriminating axis remains the estimator's bound type
(Proposition~\ref{prop:boundtype}). The methods separate \emph{off} the ideal
along exactly the gap structure of Section~\ref{sec:hierarchy}, independent
corroboration of the hierarchy rather than competition with it. Second, the
guarantee is proven for matched dimension, whereas the theory's anchor,
LeWorldModel ($d=192$), is strongly over-complete ($d\gg n$). In that regime the
embedding has far more dimensions than the world has latent degrees of freedom, so
driving the embedding to an isotropic Gaussian does not by itself fix \emph{which}
subspace carries the recovered latent. Whether isotropy nonetheless yields usable
\emph{approximate} identifiability there is open
(Section~\ref{sec:conclusion}). The practical consequence is a data-regime
condition: latent distances are physically meaningful only under sufficiently
broad state coverage.

\subsection{Testable predictions}
\label{sec:predictions}

The correspondence makes predictions that differ by \emph{kind} between SIGReg
and VICReg. Table~\ref{tab:predictions} collects the five most directly
measurable ones, each tied to the result that entails it. These are stated here
as theoretical consequences of the correspondence; their empirical evaluation is
left to separate work.

\begin{table}[t]
\centering
\small
\caption{Testable predictions. Each ties a qualitative SIGReg-vs-VICReg
contrast to a single measurement and a source result. $\kappa$ is the embedding
covariance condition number; $\rhoS$ a Spearman rank correlation. The \emph{Tier}
column uses the three-way vocabulary of Table~\ref{tab:status}. \emph{Exact}: the
prediction reads directly off an exact-in-scope result. \emph{Asymptotic}: it follows
from the $\mathcal{O}(\delta^2)$ scaling of the anisotropy term. \emph{Structural}: it is
a behavioural consequence of the correspondence rather than a deductive corollary,
and is correspondingly the most in need of empirical test. TP5 additionally inherits
the identifiability precondition of \S\ref{sec:identifiability}, which is proven only
for matched dimension.}
\label{tab:predictions}
\setlength{\tabcolsep}{4pt}
\begin{tabular}{p{0.045\linewidth} p{0.35\linewidth} p{0.145\linewidth} p{0.095\linewidth} p{0.20\linewidth}}
\toprule
ID & Prediction (SIGReg vs.\ VICReg) & Measurable & Tier & Source \\
\midrule
TP1 & $\kappa^{\mathrm{SIG}}\!\approx\!1 \ll \kappa^{\mathrm{VR}}$ (embedding isotropy) & condition number & \textbf{exact} & Thm.~\ref{thm:correspondence}(i), Prop.~\ref{prop:sigzero} \\
TP2 & higher $\rhoS$ between planning cost and success & Spearman $\rhoS$ & structural & Thm.~\ref{thm:multistep} \\
TP3 & $\kappa$ correlates with degradation under VICReg; flat under SIGReg & $\kappa$ vs.\ success & asymptotic & Cor.~\ref{cor:dual} \\
TP4 & better discrimination of nearby goals & goal-pair success & structural & Prop.~\ref{prop:pragmatic} \\
TP5 & latent goal distance tracks physical distance more tightly & $\rhoS(d_{\mathrm{lat}},d_{\mathrm{phys}})$ & structural & Prop.~\ref{prop:pragmatic} $+$ \S\ref{sec:identifiability} \\
\bottomrule
\end{tabular}
\end{table}

\section{Conclusion and Future Work}
\label{sec:conclusion}

We have argued that the anti-collapse regulariser is not an implementation detail
but the component that determines whether a JEPA world model's objective is a
valid Active Inference free energy. Organising VICReg, LogDet, PairDist, and
SIGReg by a three-source prior-miscalibration gap shows that bound \emph{type}
governs whether the AIF surprise bound survives, and that SIGReg is the first
non-contrastive regulariser that is simultaneously safe and gap-free. The
correspondence theorem makes this precise: under the constant-noise model and
successful SIGReg enforcement, the gap is exactly zero, the objective is an exact
information bottleneck, the bound is preserved, and the latent goal cost is an
exact isotropic proxy for the pragmatic value, while VICReg leaves an irreducible
$\mathcal{O}(\delta^2)$ gap. The result extends to multi-step planning, ensemble epistemic
value, and a learned-policy regime that makes action selection an exact amortised
inference.

\subsection{Limitations and scoped open problems}
\label{sec:limitations}

\paragraph{Status of the main results.} The correspondence results hold
\emph{exactly} only within their stated idealising assumptions: the
constant-noise encoder model (Fact~2) and population-limit SIGReg enforcement
($p_\Z=\Norm{0}{\tfrac cd I_d}$ with $M,N\to\infty$). Table~\ref{tab:status}
records, for each main result, what is exact under those assumptions, what
degrades at finite sample, and what is structural (motivated but approximate). It
is the single reference for the exact / asymptotic / structural distinction used
throughout; statements elsewhere inherit these labels, with the finite-sample
chain (Proposition~\ref{prop:rate}) and the autoregressive/MPC bounds
(Theorem~\ref{thm:multistep}) deferred to Appendix~\ref{app:proofs}.

\begin{table}[t]
\centering
\small
\setlength{\tabcolsep}{5pt}
\renewcommand{\arraystretch}{1.25}
\caption{Status of the main results. The exact-in-scope results
(Theorem~\ref{thm:correspondence}, Corollary~\ref{cor:dual},
Propositions~\ref{prop:sigzero} and~\ref{prop:pragmatic}) are not claimed to hold
unconditionally: they are exact given Fact~2 and successful enforcement, and
degrade gracefully at finite $(M,N)$ at the rate of
Proposition~\ref{prop:rate} (Corollary~\ref{cor:graceful}). The structural results
are labelled as such at each occurrence and are not exact identities.}
\label{tab:status}
\begin{tabular}{@{}p{0.30\linewidth}p{0.30\linewidth}p{0.34\linewidth}@{}}
\toprule
Result & Establishes & Status under the stated assumptions \\
\midrule
Prop.~\ref{prop:gap} (three-source gap) & algebraic gap identity &
\textbf{Exact} (algebraic; no further assumptions) \\
Prop.~\ref{prop:sigzero} (SIGReg eliminates gap) & $\Delta=0$, $h(\Z)=\hstar(c,d)$ &
\textbf{Exact}, conditional on successful enforcement $p_\Z=\Norm{0}{\tfrac cd I_d}$ \\
Thm.~\ref{thm:correspondence} (single-step) & exact IB/AIF decomposition; $F\ge-\ln p(x)$ &
\textbf{Exact within} the constant-noise model and population limit; finite-sample degradation bounded by Prop.~\ref{prop:rate} \\
Cor.~\ref{cor:dual} (dual tightening) & zero $\mathcal{O}(\delta^2)$ Gaussian-bridge error &
\textbf{Exact} under isotropy (same assumptions as Thm.~\ref{thm:correspondence}) \\
Prop.~\ref{prop:rate} (finite-sample rate) & $\mathcal{O}(M^{-2\alpha/(d-1)}+N^{-1/2})$ &
\textbf{Asymptotic / order-of-magnitude}; the quantitative Cram\'er--Wold step rests on a multivariate Berry--Esseen argument with implicit $d$-dependent constants -- the one link not yet self-contained (Remark~\ref{rem:rate}) \\
Prop.~\ref{prop:pragmatic} (pragmatic value) & latent goal cost as pragmatic-value proxy &
\textbf{Exact} under SIGReg (Thm.~\ref{thm:correspondence} assumptions); $\mathcal{O}(\delta^2)$-distorted under VICReg \\
Thm.~\ref{thm:multistep} (multi-step / MPC) & (i)~teacher-forced exact; (ii)--(iii)~autoregressive/MPC bounded &
(i)~\textbf{Exact}; (ii)--(iii)~\textbf{bounded}, compounding controlled by the predictor Lipschitz constant and replanning interval $m$ (exact per-step at $m=1$) \\
Prop.~\ref{prop:o2} (state-epistemic value) & coverage term absent in JEPA planners &
\textbf{Structural}: the one EFE term no current JEPA world model computes, only partially proxied by a policy ensemble \\
Learned-policy regime (\S\ref{sec:method-policy}) & MPC maps 3/4 EFE terms; learned-policy 4/4, the state-epistemic term only partially &
\textbf{Structural} correspondence; pointwise (MPC) vs.\ distributional (learned-policy) guarantees \\
\bottomrule
\end{tabular}
\end{table}

The correspondence rests on idealised assumptions, which we gather here as
deliberately scoped open problems rather than as concessions. For each we note
whether it bounds the \emph{qualitative} result (the direction of the AIF bound
and the existence of the correspondence) or only the \emph{quantitative}
constants. In every case below it is the latter.

\paragraph{The constant-noise model is an interpretive device, not a claim about
deployed systems.} Deployed JEPAs inject no observation noise; a strictly
deterministic encoder makes $I(\Z;X)$ degenerate. The constant additive noise of
Fact~2 is the standard remedy \citep{kolchinsky2019caveats} that makes the mutual
information well-defined, with the deterministic encoder recovered as the
$\sigma_{\mathrm{noise}}\!\to\!0$ reading. The construction can equally be framed
not as an a~priori assumption but as an asymptotic consequence of the data
distribution together with SIGReg convergence: once the embeddings are driven to
the isotropic Gaussian, the conditional entropy is constant by construction.
Crucially, the whole device is consistent with the non-generative
 stance \citep{lecun2022path}: SIGReg isotropy
supplies the implicit Gibbs prior, the density a trained JEPA learns
``secretly'' \citep{balestriero2025gaussian}, without a $\log Z$ term in the
objective. This bounds none of the qualitative results; it is the lens through
which the information-theoretic quantities are defined.

\paragraph{The finite-sample rate has a qualitative-only step.} The convergence
rate of Proposition~\ref{prop:rate} chains four results, one of which, the
quantitative Cram\'er--Wold conversion for the non-independent directions SIGReg
uses, is classical only in its qualitative form (Remark~\ref{rem:rate}). A fully
explicit finite-$d$ constant requires a multivariate Berry--Esseen / Stein
argument \citep{bonis2020stein,vershynin2018high} and is deferred to future work.
This bounds only the \emph{constant} in the rate, not the fact that the gap
vanishes in the population limit; under approximate enforcement the slack degrades
gracefully and linearly (Corollary~\ref{cor:graceful}).

\paragraph{Compounding error is controlled by deployment, not by the regulariser.}
The autoregressive bound of Theorem~\ref{thm:multistep}(ii) carries a term
$\mathcal{O}(L_P^{2H})$ in the predictor Lipschitz constant, which is vacuous for an
expansive predictor over a long horizon. This is resolved in practice by the
deployment regime rather than the theory: under one-step-replanning MPC
(Theorem~\ref{thm:multistep}(iii), the PLDM/LeWorldModel default) the effective
horizon is one and both error terms vanish, and spectral normalisation of the
predictor bounds $L_P$ directly. The qualitative executed-trajectory
correspondence is therefore exact under standard deployment.

\paragraph{The state-epistemic gap is a structural finding, not a defect.} The
coverage term $h(\Z_\tau\mid\pi)-\Ce$ (Proposition~\ref{prop:o2}) is computed by
no current JEPA world model and is only partially proxied by the policy ensemble;
separating behavioural multi-modality from reducible uncertainty is open. This
does not bound the correspondence; it is its sharpest prediction, naming exactly
what a JEPA agent must add to become a complete active-inference agent.

\paragraph{The over-complete identifiability regime is unproven.} The
linear-identifiability guarantee underwriting the latent-distance reading
(Section~\ref{sec:identifiability}) is proven only for matched dimension, whereas
the theory's anchor operates with $d\gg n$
\citep{klindt2026lejepa}. Whether SIGReg isotropy still yields usable
\emph{approximate} identifiability in that over-complete regime is open. Settling
it would require two things the present paper does not supply: a broad-coverage
data regime, since latent distances are physically meaningful only when the data
actually explore the state space, and a dedicated identifiability diagnostic. Beyond this, scaling to large
video-encoder dimensions ($d\ge1024$) and the calibration quality of small
($K\!=\!4$--$8$) ensembles are open empirical questions.

\paragraph{Pragmatic value is a preference, not a cost-to-go.}
Proposition~\ref{prop:pragmatic} identifies the latent goal cost with AIF
\emph{pragmatic value}, a log goal prior over states. \citet{wang2023optimal} show
that the optimal goal-reaching \emph{cost-to-go} is instead a quasimetric,
asymmetric because reachability is irreversible, so no symmetric Euclidean distance
represents it in general. The two objects are distinct and the correspondence is
unaffected: preferences over states are symmetric by construction, and the
asymmetry of reachability is carried by the transition kernel inside the expected
free energy, not by the terminal cost. The residue is a deployment caveat.
Short-horizon MPC uses the terminal cost as an implicit surrogate for the
cost-to-go beyond its horizon, and there a symmetric latent distance is a good
proxy only when the dynamics are approximately reversible. \citet{terver2025drives}
find an $L^2$ terminal cost strongest empirically, but in navigation and
manipulation benchmarks whose dynamics are largely reversible. Whether an
asymmetric goal term is required under irreversible dynamics is open.

\paragraph{Approximate Gaussianity does not blunt the distinction.}
\citet{balestriero2025gaussian} prove that any successfully trained JEPA already
learns approximately Gaussian embeddings. The distinction nonetheless matters, for
three reasons the theory makes precise. \emph{First}, ``approximately Gaussian'' is
not ``exactly Gaussian,'' and the residual is exactly what $\Delta$ measures;
because the AIF bound is binary in the estimator's \emph{type}
(Proposition~\ref{prop:boundtype}), VICReg is not guaranteed to preserve it even when
the residual is small, while SIGReg penalises that residual directly
(Corollary~\ref{cor:graceful}). \emph{Second}, their result characterises the
optimum, whereas the hierarchy and Proposition~\ref{prop:rate} govern the
\emph{path} to it. Suggestive here is LeWorldModel's smooth two-term training
against PLDM's noisier multi-term dynamics \citep{maes2026leworldmodel}, though
because the two systems differ in architecture and in the number of loss terms,
not in the regulariser alone, this is corroborating rather than confirmatory;
isolating the regulariser's contribution requires a controlled SIGReg-vs-VICReg
ablation under a fixed architecture. \emph{Third}, the bridge is valuable
independently of the comparison: even were the two regularisers empirically
indistinguishable, the correspondence would still supply the validity criterion,
the principled epistemic-value term, and the transfer of tools between active
inference and the JEPA world-model literature
(Appendix~\ref{app:positioning}).

\paragraph{Outlook.} The contribution of this paper is theoretical and, within
the stated scope, complete: it establishes that the anti-collapse regulariser
determines whether a JEPA world model's objective is a valid active-inference
free energy, makes that correspondence exact under SIGReg, traces it through
multi-step planning and a learned-policy regime, and isolates the single
expected-free-energy term, the state-epistemic value, that no current JEPA
world model computes. The predictions of Section~\ref{sec:predictions} differ by
kind between SIGReg and VICReg and are therefore refutable; their empirical
evaluation is left to separate work. Establishing the
theory and stating its falsifiable consequences is the purpose served here. The
experimental validation, and any consequent revision of the limitations above,
will follow separately.

\bibliographystyle{plainnat}
\bibliography{references}

\newpage
\appendix
\section*{Appendix}
\addcontentsline{toc}{section}{Appendix}
\noindent The appendices are supplementary. Appendix~\ref{app:proofs} gives the
full proof development for the multi-step and expected-free-energy results
summarised in Sections~\ref{sec:method-efe} and~\ref{sec:method-efe-decomp}; environment numbering here is internal
to the appendix (e.g.\ Proposition~\ref{app:prop:tf}) and is cross-referenced from
the main text where relevant. Appendix~\ref{app:figures} collects the schematic
figures, Appendix~\ref{app:positioning} expands the positioning discussion, and
Appendix~\ref{app:lean} reports the formal verification of the paper's algebraic
core in the Lean~4 theorem prover.

\section{Full Proof Development: Multi-Step Correspondence and EFE Decomposition}
\label{app:proofs}

This appendix expands the multi-step results that Section~\ref{sec:method-efe}
states in compressed form. Throughout, $f_\phi$ is a deterministic encoder and
$P_\xi$ a deterministic predictor under the constant-noise model (Background,
Fact~2): $z_{t+\tau}=f_\phi(x_{t+\tau})+\epsilon_{t+\tau}$ with
$\epsilon_{t+\tau}\overset{\text{i.i.d.}}{\sim}\Norm{0}{\varepsilon^2 I_d}$
independent across steps and of $\phi$. We write $\sigma^2$ for the dynamics
noise variance, $\Ce=\tfrac d2\ln(2\pi e\,\varepsilon^2)$ for the conditional
entropy, and $C_{\mathrm{KL}}=\tfrac d2(\varepsilon^2/\sigma^2-1-
\ln(\varepsilon^2/\sigma^2))\ge0$ for the per-step Gaussian-bridge constant
(the $\varepsilon^2$ here is the noise \emph{variance}, written $\varepsilon$
in \eqref{eq:bridge} and $\sigma^2_{\mathrm{noise}}$ in Fact~2). The
single-step correspondence (main-text Theorem~\ref{thm:correspondence}) and the
dual-tightening corollary (Corollary~\ref{cor:dual}) are taken as given; under
SIGReg enforcement the per-step KL--MSE bridge \eqref{eq:bridge} is exact
(zero anisotropy error).

The development proceeds through three operational regimes: teacher-forcing
(\S\ref{app:tf}), autoregressive rollout (\S\ref{app:ar}), and model-predictive
control (\S\ref{app:mpc}), consolidated in Theorem~\ref{app:thm:multistep}, and
then through the term-by-term EFE decomposition (\S\ref{app:efe}). Together these
results establish the multi-step claims summarised in
Section~\ref{sec:method-efe} of the main text; the supporting lemmas below
(factorisation, Jacobian correction, compounding error) are the intermediate
steps that the main text omits.

\subsection{Stage 1: Teacher-forced multi-step exactness}
\label{app:tf}

Under teacher-forcing, the predictor at each step receives the true encoded
observation, $\hat z^{\mathrm{TF}}_{t+\tau}\coloneqq P_\xi(z_{t+\tau-1},
a_{t+\tau-1})$, so each step is independently grounded in data. Define the
per-step posterior $q_\phi^{(\tau)}=\Norm{f_\phi(x_{t+\tau})}{\varepsilon^2 I_d}$
and prior $p_\xi^{(\tau)}=\Norm{\hat z^{\mathrm{TF}}_{t+\tau}}{\sigma^2 I_d}$, and
write $\mathcal{L}^{(H)}_{\mathrm{MSE}}\coloneqq\sum_{\tau=1}^{H}\|\hat
z^{\mathrm{TF}}_{t+\tau}-z_{t+\tau}\|^2$ for the teacher-forced multi-step MSE.

\begin{alemma}[Teacher-forced factorisation]
\label{app:lem:tf}
For any fixed trajectory $(\mathbf o,\mathbf a)$, the teacher-forced joint
distributions factorise:
$q^{\mathrm{TF}}(z_{t+1:t+H})=\prod_{\tau=1}^{H}q_\phi^{(\tau)}$ and
$p^{\mathrm{TF}}(z_{t+1:t+H})=\prod_{\tau=1}^{H}p_\xi^{(\tau)}$.
\end{alemma}

\begin{proof}
Under the constant-noise model the only randomness in $z_{t+\tau}$ is the
independent noise $\epsilon_{t+\tau}$; conditioning on $\mathbf o$ fixes the
means, so $z_{t+1},\dots,z_{t+H}$ are conditionally independent and the joint
density is the product of marginals. For the prior, teacher-forcing conditions
$p_\xi^{(\tau)}$ on the realised $z_{t+\tau-1}$, making each a fixed Gaussian with
intrinsic noise $\sigma^2$ independent across steps; the joint therefore
factorises. (The \emph{unconditional} joint would not factorise, since it chains
the dynamics, but the operationally relevant conditioning is on the realised
encoder outputs, since the loss is computed after encoding all observations.)
\end{proof}

\begin{aproposition}[Teacher-forced multi-step KL--MSE equivalence]
\label{app:prop:tf}
Under SIGReg enforcement ($\Sigma=\tfrac cd I_d$), the joint trajectory KL equals
the scaled teacher-forced MSE plus a constant, exactly:
\begin{equation}
\KL{q^{\mathrm{TF}}}{p^{\mathrm{TF}}}
= \sum_{\tau=1}^{H}\KL{q_\phi^{(\tau)}}{p_\xi^{(\tau)}}
= \frac{1}{2\sigma^2}\,\mathcal{L}^{(H)}_{\mathrm{MSE}} + H\,C_{\mathrm{KL}} .
\label{app:eq:tf}
\end{equation}
Consequently $\arg\min_\xi \KL{q^{\mathrm{TF}}}{p^{\mathrm{TF}}}=\arg\min_\xi
\mathcal{L}^{(H)}_{\mathrm{MSE}}$ and the gradients coincide up to the factor
$1/2\sigma^2$, for every $\varepsilon>0$.
\end{aproposition}

\begin{proof}
By Lemma~\ref{app:lem:tf} the two joints are products, and the KL between product
measures is additive: expanding $\ln\!\big(\prod_\tau q^{(\tau)}/\prod_\tau
p^{(\tau)}\big)=\sum_{\tau'}\ln(q^{(\tau')}/p^{(\tau')})$ and integrating, each
summand depends only on its own coordinate, so integrating out the others gives
unity and leaves $\sum_\tau\KL{q_\phi^{(\tau)}}{p_\xi^{(\tau)}}$
\citep[Thm.~2.5.3]{cover2006elements}. Each per-step term is a KL between
isotropic Gaussians; under SIGReg the variance ratios are equal across
dimensions, so Corollary~\ref{cor:dual} makes \eqref{eq:bridge} exact:
$\KL{q_\phi^{(\tau)}}{p_\xi^{(\tau)}}=\tfrac{1}{2\sigma^2}\|\hat
z^{\mathrm{TF}}_{t+\tau}-z_{t+\tau}\|^2+C_{\mathrm{KL}}$. Summing over $\tau$ gives
\eqref{app:eq:tf}; the constant $H\,C_{\mathrm{KL}}$ is independent of $\xi$, so
the minimisers and gradients follow.
\end{proof}

Summing the per-step free energy $F^{(\tau)}_{\mathrm{SIG}}=\tfrac{1}{2\sigma^2}
\|\hat z^{\mathrm{TF}}_{t+\tau}-z_{t+\tau}\|^2+C_{\mathrm{KL}}-\Istar$ over the
horizon gives the teacher-forced multi-step free energy
$F^{(H)}_{\mathrm{TF}}=\tfrac{1}{2\sigma^2}\mathcal{L}^{(H)}_{\mathrm{MSE}}
-H\Istar+H\,C_{\mathrm{KL}}$, which is optimisation-equivalent to
$\mathcal{L}^{(H)}_{\mathrm{MSE}}$ and preserves the bound at the trajectory
level: $F^{(H)}_{\mathrm{TF}}\ge-\sum_\tau\ln p(x_{t+\tau})$. Under VICReg, the
per-step bridge carries an $\mathcal{O}(\delta^2)$ error and the cumulative gap grows
linearly, $|\KL{q^{\mathrm{TF}}}{p^{\mathrm{TF}}}-\tfrac{1}{2\bar\sigma^2}
\mathcal{L}^{(H)}_{\mathrm{MSE}}|\le H\cdot\overline E^{(\mathrm{VR})}$ -- a
``horizon tax'' that vanishes identically under SIGReg.

\subsection{Stage 2: Autoregressive rollout}
\label{app:ar}

At planning time the predictor is unrolled on its own output, $\hat
z^{\mathrm{AR}}_{t+\tau}=P_\xi(\hat z^{\mathrm{AR}}_{t+\tau-1},a_{t+\tau-1})$ with
$\hat z^{\mathrm{AR}}_t=z_t$. Two effects break Stage-1 exactness: the AIF
generative model becomes a Markov chain (not a product), and the planning MSE
evaluates predictions under the rollout distribution. Let $e_\tau\coloneqq
z_{t+\tau}-\hat z^{\mathrm{TF}}_{t+\tau}$, $\delta_\tau\coloneqq\hat
z^{\mathrm{AR}}_{t+\tau}-\hat z^{\mathrm{TF}}_{t+\tau}$ (with $\delta_1=0$), and
$J_\tau\coloneqq\nabla_z P_\xi(z,a_\tau)\big|_{z=z_{t+\tau}}$ the predictor
Jacobian.

\begin{aproposition}[Markov-chain trajectory KL and the Jacobian correction]
\label{app:prop:jac}
With product posterior $q$ and Markov-chain prior $p^{\mathrm{MC}}(z_{1:H})=
\prod_\tau\Norm{P_\xi(z_{\tau-1},a_{\tau-1})}{\sigma^2 I_d}$, under SIGReg
enforcement,
\begin{equation}
\KL{q}{p^{\mathrm{MC}}}
= \underbrace{\tfrac{1}{2\sigma^2}\mathcal{L}^{(H)}_{\mathrm{MSE}}+H\,C_{\mathrm{KL}}}
   _{=\,\KL{q^{\mathrm{TF}}}{p^{\mathrm{TF}}}}
+ \frac{1}{2\sigma^2}\sum_{\tau=2}^{H}\Phi_\tau,
\qquad
\Phi_\tau = \varepsilon^2\|J_{\tau-1}\|_F^2 + \mathcal{O}(\varepsilon^3),
\label{app:eq:jac}
\end{equation}
where $\Phi_\tau\ge0$ in the small-noise regime. Hence
$\KL{q}{p^{\mathrm{MC}}}\ge\KL{q^{\mathrm{TF}}}{p^{\mathrm{TF}}}$.
\end{aproposition}

\begin{proof}
Expanding the KL with the product $q$ against the chained $p^{\mathrm{MC}}$, each
summand depends only on $(z_{\tau-1},z_\tau)$; integrating out the rest gives
$\sum_\tau\E_{z_{\tau-1}\sim q^{(\tau-1)}}\!\big[\KL{q^{(\tau)}}{p(\cdot\mid
z_{\tau-1})}\big]$. Under SIGReg both arguments are isotropic Gaussians, so the
inner KL equals $\tfrac{1}{2\sigma^2}\|z_{t+\tau}-P_\xi(z_{\tau-1},
a_{\tau-1})\|^2+C_{\mathrm{KL}}$. Separating the $\tau=1$ term (where $z_0=z_t$ is
fixed) and adding/subtracting the teacher-forced MSE yields the decomposition
with $\Phi_\tau=\E_{z_{\tau-1}}\|z_{t+\tau}-P_\xi(z_{\tau-1})\|^2-
\|z_{t+\tau}-P_\xi(z_{t+\tau-1})\|^2$. Taylor-expanding $P_\xi$ about
$z_{t+\tau-1}$ with $z_{\tau-1}=z_{t+\tau-1}+\varepsilon\eta$,
$\eta\sim\Norm{0}{I_d}$, the linear term vanishes in expectation and
$\E\|J_{\tau-1}\eta\|^2=\tr(J_{\tau-1}^\top J_{\tau-1})=\|J_{\tau-1}\|_F^2$, giving
$\Phi_\tau=\varepsilon^2\|J_{\tau-1}\|_F^2+\mathcal{O}(\varepsilon^3)$. The correction is
$\mathcal{O}(\varepsilon^2/\sigma^2)$ -- negligible in the deterministic-encoder limit --
and under SIGReg's isotropic noise it is the unweighted Frobenius norm (under
VICReg it would be the anisotropy-weighted $\tr(J^\top\Sigma_\varepsilon J)$).
\end{proof}

\begin{aproposition}[Compounding error]
\label{app:prop:comp}
If $P_\xi(\cdot,a)$ is $L_P$-Lipschitz in its first argument, the autoregressive
error $\epsilon^{\mathrm{AR}}_\tau\coloneqq\|\hat z^{\mathrm{AR}}_{t+\tau}-
z_{t+\tau}\|$ obeys $\epsilon^{\mathrm{AR}}_\tau\le L_P\,
\epsilon^{\mathrm{AR}}_{\tau-1}+\epsilon^{\mathrm{TF}}_\tau$, hence
$\epsilon^{\mathrm{AR}}_\tau\le\sum_{s=1}^{\tau}L_P^{\tau-s}
\epsilon^{\mathrm{TF}}_s$, and the planning MSE satisfies
$\mathcal{L}^{(H)}_{\mathrm{plan}}=\mathcal{L}^{(H)}_{\mathrm{MSE}}+
\Delta^{(H)}_{\mathrm{comp}}$ with $\Delta^{(H)}_{\mathrm{comp}}=
-2\sum_\tau\langle e_\tau,\delta_\tau\rangle+\sum_\tau\|\delta_\tau\|^2$. The
error saturates ($L_P<1$), grows linearly ($L_P=1$), or grows exponentially
($L_P>1$) with the horizon.
\end{aproposition}

\begin{proof}
The recurrence is the triangle inequality applied to $\hat z^{\mathrm{AR}}_{t+\tau}
=P_\xi(\hat z^{\mathrm{AR}}_{t+\tau-1},a_{\tau-1})$ together with
$L_P$-Lipschitzness; unrolling it gives the geometric bound. Writing
$\hat z^{\mathrm{AR}}_{t+\tau}-z_{t+\tau}=\delta_\tau-e_\tau$ and squaring yields
the decomposition of $\mathcal{L}^{(H)}_{\mathrm{plan}}$, with $\|\delta_\tau\|\le
L_P\,\epsilon^{\mathrm{AR}}_{\tau-1}$. The three regimes follow from the geometric
series. SIGReg does not control $L_P$; it is set by the predictor and its
regularisation.
\end{proof}

Chaining Propositions~\ref{app:prop:jac}--\ref{app:prop:comp} relates the AIF
trajectory complexity to the planning MSE,
$\KL{q}{p^{\mathrm{MC}}}=\tfrac{1}{2\sigma^2}\mathcal{L}^{(H)}_{\mathrm{plan}}
+H\,C_{\mathrm{KL}}+\tfrac{1}{2\sigma^2}(\Phi^{(H)}-\Delta^{(H)}_{\mathrm{comp}})$,
so the absolute gap is bounded by the Jacobian term plus the compounding term.

\subsection{Stage 3: Model-predictive control and the executed trajectory}
\label{app:mpc}

Under MPC with replanning interval $m$, the $H$-step rollout is only a scoring
function; the agent executes $m$ actions, then re-encodes a true observation. The
executed trajectory is thus grounded every $m$ steps, and the effective
autoregressive horizon within each epoch is $m$, not $H$.

\begin{acorollary}[Exact executed-trajectory correspondence at $m=1$]
\label{app:cor:mpc}
Under SIGReg enforcement and MPC with $m=1$ (the PLDM/LeWorldModel default), each
executed step resets at a true embedding, so $\delta_1=0$ and the Jacobian sum is
empty: $\Delta^{(1)}_{\mathrm{comp}}=0$ and $\Phi^{(1)}=0$. Each executed step
satisfies the exact single-step bridge,
\begin{equation}
\KL{q^{(n)}}{p^{\mathrm{MC},(n)}}
= \frac{1}{2\sigma^2}\big\|z_{t_n+1}-\hat z_{t_n+1}\big\|^2 + C_{\mathrm{KL}},
\label{app:eq:mpc}
\end{equation}
and over $N$ epochs the executed-trajectory complexity is the scaled sum of
per-step MSEs plus $N\,C_{\mathrm{KL}}$ and an environment-determined inter-epoch
term $\Gamma_{\mathrm{inter}}$ that is independent of the regulariser.
\end{acorollary}

\begin{proof}
With $m=1$ each epoch is a single prediction grounded in $z_{t_n}$, so the
autoregressive and teacher-forced predictions coincide and both Stage-2
corrections vanish; \eqref{app:eq:mpc} is Proposition~\ref{app:prop:tf} at $H=1$
(equivalently Corollary~\ref{cor:dual}). Summing over epochs and adding the
inter-epoch term gives the trajectory statement. $\Gamma_{\mathrm{inter}}$
reflects how much the true latent state at one epoch predicts the next; it is a
property of the data, present regardless of regulariser or planner.
\end{proof}

For action selection what matters is ranking fidelity, not the absolute gap. Two
action sequences are correctly ordered whenever their planning-cost difference
exceeds their differential correction; under SIGReg each per-step contribution is
an exact KL proxy, so there is zero per-step ranking distortion, whereas under
VICReg an additional action-dependent $\mathcal{O}(H\delta^2)$ distortion is present.

\subsection{Why the state-epistemic term has no JEPA counterpart}
\label{app:state-epistemic}

Proposition~\ref{prop:o2} states that the state-epistemic value is the one
expected-free-energy term that no current JEPA world model computes. Because this
is the paper's central \emph{structural} claim, and because it is easy to
mistake SIGReg's entropy guarantee for a resolution of it, we develop the point
in full here. The argument has four steps: the EFE contains two distinct
epistemic drives; only one has a JEPA counterpart; SIGReg's marginal-entropy
guarantee does not supply the other; and a learned policy closes the gap only
partially.

\paragraph{Step 1: the EFE contains two distinct epistemic drives.}
Written in full, the expected free energy of a policy contains two information-gain
terms that are routinely conflated because both are called ``epistemic.'' The
first is the \emph{parameter} information gain $I_q(\xi;Z_\tau\mid Z_{\tau-1},
a_{\tau-1})$: how much executing the policy is expected to reduce uncertainty
about the model parameters $\xi$. This is the drive to act where the model is
unsure of its own dynamics. The second is the \emph{state} epistemic value
$I_q(Z_\tau;X_\tau\mid\pi)$: under the constant-noise model
(Proposition~\ref{prop:o2}) it equals $h(\Z_\tau\mid\pi)-\Ce$, the entropy of the
distribution of \emph{embeddings the policy would visit}, minus the constant
observation-noise entropy. These measure different things. The parameter term is
about epistemic uncertainty over $\xi$; the state term is about the
\emph{coverage} or diversity of the visited-state distribution. A policy that
confines the agent to a small region of embedding space has low
$h(\Z_\tau\mid\pi)$ and hence low state-epistemic value; a policy that fans the
agent out across the space has high $h(\Z_\tau\mid\pi)$ and high state-epistemic
value. The latter is precisely the intrinsic drive to explore that distinguishes
an active-inference agent from a pure goal-reacher, independent of any goal.

\paragraph{Step 2: only the parameter drive has a JEPA counterpart.}
The parameter information gain \emph{does} have a computable JEPA proxy. With an
ensemble of $K$ predictors $\{P^k_\xi\}$, a second-order expansion gives
$I_q(\xi;Z_\tau\mid Z_{\tau-1},a_{\tau-1})\approx\tfrac{1}{2\sigma^2}\sum_j
\Var_k[P^k_\xi(z_{\tau-1},a_{\tau-1})_j]$, which is exactly PLDM's
ensemble-disagreement (uncertainty) cost. So this epistemic term is present, at
least approximately, in JEPA planners that carry an ensemble penalty. The state
epistemic value is different: it is a property of the visited-state
\emph{distribution}, not of predictor disagreement. A JEPA planner scores a
candidate trajectory by its goal cost plus, optionally, an ensemble term. The
goal cost $\|\z_\tau-\z_g\|^2$ is minimised by \emph{reaching} $\z_g$ and is
indifferent to how much of the space the trajectory explores on the way; the
ensemble term measures parameter uncertainty, not state diversity. No quantity in
the objective increases when $h(\Z_\tau\mid\pi)$ increases, so there is nothing
whose minimisation rewards coverage. The term is simply absent.

\paragraph{Step 3: SIGReg's marginal entropy does not supply the conditional
entropy.} This is the step most likely to mislead. SIGReg \emph{does} maximise an
entropy: it drives the \emph{marginal} embedding entropy to its ceiling,
$h(\Z)=\hstar(c,d)$ (Proposition~\ref{prop:sigzero}). It is therefore tempting to
conclude that SIGReg has already supplied the coverage term. It has not, because
the two entropies are different objects:
\begin{itemize}[leftmargin=1.6em,itemsep=2pt,topsep=2pt]
\item $h(\Z)$ is the entropy of embeddings \emph{averaged over the full data
distribution}, evaluated at \emph{training time}; it is a property of the
\emph{encoder}. SIGReg maximises it by construction.
\item $h(\Z_\tau\mid\pi)$ is the entropy of the embeddings induced by \emph{one
policy's trajectory}, evaluated at \emph{planning time}; it is a property of the
\emph{policy's behaviour}. SIGReg never sees it.
\end{itemize}
SIGReg makes the overall space maximally informative, but it says nothing about
whether a particular policy chooses to explore that space or to huddle in one
corner of it. The only link is an inequality: for any policy,
$h(\Z_\tau\mid\pi)\le h(\Z)=\hstar(c,d)$, since conditioning cannot increase
entropy. Thus SIGReg establishes the \emph{ceiling} on state-epistemic value,
since no policy's coverage can exceed $\Istar$, but it neither computes
$h(\Z_\tau\mid\pi)$ nor inserts it into any planner's objective. As the theory
puts it, SIGReg has already ``spent'' the full entropy budget on the marginal; the
question the state-epistemic term asks, namely whether \emph{this} policy
concentrates the trajectory into a subspace, is left entirely open. This is the same
marginal-versus-conditional distinction flagged in
Section~\ref{sec:method-efe-decomp}: enforced single-step informativeness $I(Z;X)=\Istar$
is a guarantee about the encoder; the multi-step state-epistemic coverage term is
a property of the policy, and the former does not imply the latter.

\paragraph{Step 4: a learned policy closes the gap only partially.}
Replacing MPC with a learned goal-conditioned policy, and training an
\emph{ensemble} of such policies, yields a signal that is \emph{correlated} with
the state-epistemic value: the disagreement among the policy heads about which
action to take captures uncertainty about the optimal action, which the dynamics
ensemble alone cannot. This is why the learned-policy regime improves the
term-by-term accounting, from three of the four EFE terms mapped under MPC to all
four accounted for, with the state-epistemic term mapped only partially. The closure is partial, not
complete, for a precise reason: the policy-ensemble disagreement
$\mathrm{tr}(V^\pi_\tau)$ conflates several sources of variance (genuine
multi-modality of good actions, value uncertainty, and state-coverage
uncertainty), and isolating the component that actually tracks
$h(\Z_\tau\mid\pi)$ from the rest is itself unresolved. The residual is the open quantity
$\Delta_{\mathrm{epistemic}}$: the learned policy supplies a \emph{proxy
correlated with} the coverage term, not the term itself. The gap narrows; it does
not close. SIGReg sharpens this further by making the conditioning space
isotropic, so that the policy ensemble's sensitivity to the latent state is
direction-independent rather than dominated by high-variance dimensions, but this
calibrates the proxy without making it exact. The state-epistemic value therefore
remains the primary structural gap between AIF and JEPA planning, and, as
Section~\ref{sec:method-efe-decomp} stresses, this is a positive finding: it names
exactly the quantity a JEPA agent would have to add to its planning objective to
become a complete active-inference agent.

\begin{atheorem}[SIGReg--AIF multi-step correspondence; formal restatement of Theorem~\ref{thm:multistep}]
\label{app:thm:multistep}
\textnormal{\itshape (Formal statement and proof of Theorem~\ref{thm:multistep},
Section~\ref{sec:method-efe}.)}\\
Under the hypotheses of Theorem~\ref{thm:correspondence} with an $L_P$-Lipschitz
predictor and MPC$(H,m)$: \emph{(i)} the teacher-forced multi-step MSE is an exact
proxy for the joint trajectory KL (Proposition~\ref{app:prop:tf}) and the bound is
preserved; \emph{(ii)} the autoregressive planning MSE approximates the AIF
trajectory complexity with error bounded by the Jacobian term
$\mathcal{O}(H\varepsilon^2/\sigma^2)$ plus the compounding term
$\mathcal{O}(L_P^{2H}\bar\epsilon^2/\sigma^2)$ (Propositions~\ref{app:prop:jac}--%
\ref{app:prop:comp}); \emph{(iii)} under MPC the same bound holds with $m$ in
place of $H$, and at $m=1$ the executed trajectory is exact
(Corollary~\ref{app:cor:mpc}); \emph{(iv)} the planning cost is ranking-faithful,
with zero per-step distortion under SIGReg. Replacing SIGReg with VICReg adds an
$\mathcal{O}(H\delta^2)$ term to every bound.
\end{atheorem}

\begin{proof}
We prove the four parts in turn, then the VICReg comparison.

\emph{(i) Teacher-forced exactness.} Under teacher-forcing the joint posterior and
prior factorise over steps (Lemma~\ref{app:lem:tf}), so the trajectory KL is the
sum of per-step KLs. By Proposition~\ref{app:prop:tf}, SIGReg isotropy makes each
per-step Gaussian bridge \eqref{eq:bridge} exact, giving
$\KL{q^{\mathrm{TF}}}{p^{\mathrm{TF}}}=\tfrac{1}{2\sigma^2}
\mathcal{L}^{(H)}_{\mathrm{MSE}}+H\,C_{\mathrm{KL}}$ with no residual. The
trajectory free energy $F^{(H)}_{\mathrm{TF}}=\tfrac{1}{2\sigma^2}
\mathcal{L}^{(H)}_{\mathrm{MSE}}-H\Istar+H\,C_{\mathrm{KL}}$ then satisfies
$F^{(H)}_{\mathrm{TF}}\ge-\sum_\tau\ln p(x_{t+\tau})$, since each per-step bound is
preserved (Proposition~\ref{prop:boundtype} at $\Delta=0$) and the bound is
additive over the horizon.

\emph{(ii) Autoregressive control.} Under free rollout the prior becomes a Markov
chain rather than a product. Proposition~\ref{app:prop:jac} gives the exact
decomposition $\KL{q}{p^{\mathrm{MC}}}=\KL{q^{\mathrm{TF}}}{p^{\mathrm{TF}}}+
\tfrac{1}{2\sigma^2}\sum_{\tau\ge2}\Phi_\tau$ with $\Phi_\tau=\varepsilon^2
\|J_{\tau-1}\|_F^2+\mathcal{O}(\varepsilon^3)$, so the Jacobian contribution is
$\mathcal{O}(H\varepsilon^2/\sigma^2)$. Separately, the rollout MSE differs from the
teacher-forced MSE by the compounding term $\Delta^{(H)}_{\mathrm{comp}}$ of
Proposition~\ref{app:prop:comp}; unrolling the Lipschitz recurrence
$\epsilon^{\mathrm{AR}}_\tau\le L_P\,\epsilon^{\mathrm{AR}}_{\tau-1}+
\epsilon^{\mathrm{TF}}_\tau$ bounds the accumulated error by
$\sum_{s\le\tau}L_P^{\tau-s}\epsilon^{\mathrm{TF}}_s$, whose square is
$\mathcal{O}(L_P^{2H}\bar\epsilon^2)$. Combining the two contributions and dividing by the
$2\sigma^2$ scaling gives the stated bound; the AIF complexity equals the planning
MSE up to these two terms plus the constant $H\,C_{\mathrm{KL}}$.

\emph{(iii) Executed-trajectory exactness under deployment.} MPC executes $m$
actions per epoch and re-encodes a true observation, so the autoregressive
analysis of part~(ii) applies within each epoch with horizon $m$ rather than $H$;
substituting $m$ for $H$ in the part~(ii) bound gives the first claim. At $m=1$
each epoch is a single prediction grounded in a true embedding, so
$\delta_1=0$ and the Jacobian sum is empty: by Corollary~\ref{app:cor:mpc} both
$\Delta^{(1)}_{\mathrm{comp}}$ and $\Phi^{(1)}$ vanish, and each executed step
satisfies the exact single-step bridge \eqref{app:eq:mpc}. Summing over the $N$
executed epochs gives an executed-trajectory complexity equal to the scaled sum of
per-step MSEs plus $N\,C_{\mathrm{KL}}$ and a regulariser-independent inter-epoch
term $\Gamma_{\mathrm{inter}}$.

\emph{(iv) Ranking fidelity.} Action selection depends only on the ordering of
planning costs across candidate sequences, not on their absolute value. Two
sequences $\pi,\pi'$ are correctly ordered whenever their planning-cost difference
exceeds the difference of their correction terms. Under SIGReg each per-step
contribution is an exact KL proxy (parts~(i),(iii)), so the per-step correction is
identically zero and the ordering of planning costs coincides with the ordering of
trajectory free energies exactly; there is no per-step ranking distortion.

\emph{VICReg comparison.} Replacing SIGReg with VICReg reopens the per-step
Gaussian-bridge error of Stage~1, which is $\mathcal{O}(\delta^2)$ in the anisotropy
$\delta$ (the eigenvalue spread of the embedding covariance; Stage~1 and
Corollary~\ref{cor:dual}). Substituting this nonzero per-step error into the
additive trajectory KL of part~(i) yields a cumulative $\mathcal{O}(H\delta^2)$ term; the
same substitution enters the part~(ii)/(iii) bounds and, because the per-step
correction is now action-dependent and nonzero, the part~(iv) ranking argument
acquires an $\mathcal{O}(H\delta^2)$ distortion. Hence every bound carries an additional
$\mathcal{O}(H\delta^2)$ term under VICReg, vanishing identically under SIGReg.
\end{proof}

\subsection{The expected-free-energy decomposition}
\label{app:efe}

This section supplies the term-by-term expectation calculations behind the
expected-free-energy reading of Section~\ref{sec:method-efe-decomp}; the
generative model \eqref{eq:genmodel} and the consolidated objective
\eqref{eq:efe-consolidated} are stated there. Throughout, the mean-field posterior
is $q(z_{1:H})=\prod_\tau\Norm{f_\phi(x_\tau)}{\varepsilon^2 I_d}$ and the per-step
EFE is $G^{(\tau)}_\pi=-I_q(Z_\tau;X_\tau\mid\pi)-\E_{q(x_\tau\mid\pi)}[\ln
p(x_\tau\mid C)]$ (epistemic $+$ pragmatic), equivalently ambiguity $+$ risk
\citep{smith2022tutorial,mazzaglia2022free}.

\begin{aproposition}[Pragmatic value]
\label{app:prop:pragmatic}
\textnormal{\itshape (Full statement and proof of Proposition~\ref{prop:pragmatic},
Section~\ref{sec:method-efe-decomp}.)}\\
The pragmatic value at step $\tau$ is
$-\E_{q(x_\tau\mid\pi)}[\ln p(x_\tau\mid C)]=\tfrac{1}{2\sigma_g^2}
(\|f_\phi(x_\tau)-z_g\|^2+d\varepsilon^2)+\tfrac d2\ln(2\pi\sigma_g^2)$, whose
policy-dependent part is the JEPA goal cost $\tfrac{1}{2\sigma_g^2}
\|z_\tau-z_g\|^2$. Under SIGReg the Euclidean cost is an exact KL proxy,
$\KL{\Norm{z_\tau}{\varepsilon^2 I_d}}{\Norm{z_g}{\varepsilon^2 I_d}}=
\tfrac{1}{2\varepsilon^2}\|z_\tau-z_g\|^2$; under VICReg it carries an
$\mathcal{O}(\delta^2)$ anisotropy discrepancy.
\end{aproposition}

\begin{proof}
Expand $\ln p(x_\tau\mid C)$ and take the expectation under $q$, using
$z_\tau=f_\phi(x_\tau)+\epsilon_\tau$, $\E[\epsilon_\tau]=0$,
$\E\|\epsilon_\tau\|^2=d\varepsilon^2$; the noise term is policy-independent and
enters the constant. The exact KL identity is the isotropic-Gaussian formula,
which under SIGReg has equal variances and hence no anisotropy correction.
\end{proof}

\begin{aproposition}[Ambiguity, risk, and the state-epistemic gap]
\label{app:prop:efe-rest}
\textnormal{\itshape (Expands the remaining EFE terms of
Section~\ref{sec:method-efe-decomp}; the state-epistemic part is
Proposition~\ref{prop:o2}.)}\\
Under the constant-noise model: \emph{(ambiguity)} $\E_{q(z_\tau\mid\pi)}
H[p(x_\tau\mid z_\tau)]=\Ce$, constant in $\pi$ and regulariser-independent;
\emph{(risk)} $\KL{q(x_\tau\mid\pi)}{p(x_\tau\mid C)}=\tfrac{1}{2\sigma_g^2}
\|z_\tau-z_g\|^2+\tfrac d2(\varepsilon^2/\sigma_g^2-1-\ln(\varepsilon^2/
\sigma_g^2))$, exact under SIGReg and the same MSE functional as the pragmatic
value up to a policy-independent constant; \emph{(state-epistemic value)}
$I_q(Z_\tau;X_\tau\mid\pi)=h(Z_\tau\mid\pi)-\Ce$, a coverage signal upper-bounded
by $\Istar$ under SIGReg and computed by no current JEPA world model.
\end{aproposition}

\begin{proof}
For ambiguity, $H[p(x_\tau\mid z_\tau)]=H[\epsilon_\tau]=\Ce$ is independent of
$z_\tau$, so the expectation is $\Ce$. For risk, the observation distributions are
in bijection with isotropic-Gaussian latent distributions, and the KL is the
standard isotropic-Gaussian formula \citep[\S10.1.1]{bishop2006pattern}. For the
state-epistemic value, $I(Z_\tau;X_\tau)=h(Z_\tau)-h(Z_\tau\mid X_\tau)$ and, since
$Z_\tau=f_\phi(X_\tau)+\epsilon_\tau$ with $\epsilon_\tau\perp X_\tau$,
$h(Z_\tau\mid X_\tau)=\Ce$; under SIGReg the marginal entropy equals
$\hstar(c,d)$ (main-text Proposition~\ref{prop:sigzero}), giving the bound.
\end{proof}

The parameter information gain maps to ensemble predictive variance: for Gaussian
dynamics the Fisher information about the predictor mean is $\sigma^{-2}I_d$, and a
second-order expansion of $\tfrac12\ln\det(I_d+\sigma^{-2}\Cov_\xi[\mu_\xi])$ gives
$I_q(\xi;Z_\tau\mid Z_{\tau-1},a_{\tau-1})\approx\tfrac{1}{2\sigma^2}\sum_j
\Var_k[P^k_\xi(z_{\tau-1},a_{\tau-1})_j]+\mathcal{O}(\sigma^{-4})$, the per-step PLDM
uncertainty cost. Substituting
Propositions~\ref{app:prop:pragmatic}--\ref{app:prop:efe-rest} together with this
term into the per-step EFE yields the consolidated objective
\eqref{eq:efe-consolidated} of Section~\ref{sec:method-efe-decomp}: the JEPA
planning cost $\Cgoal+\beta\,\Cunc$ equals $G^{\mathrm{full}}_\pi$ minus the absent
state-epistemic term, with each retained term exact under SIGReg and
$\mathcal{O}(\delta^2)$-distorted under VICReg. The learned-policy identity that closes
Section~\ref{sec:method-policy}, that minimising $\E_{\pi_\phi}[\sum_\tau
\gamma^\tau G^{\mathrm{SIG},(\tau)}_\pi]-\tfrac1\zeta H[\pi_\phi]$ equals
$\tfrac1\zeta\KL{\pi_\phi}{\pi^{*}}$ up to a constant, with Boltzmann optimum
$\pi^{*}\propto\exp(-\zeta G^{\mathrm{SIG}})$, follows by the standard
free-energy-to-KL completion, so CEM ($\zeta\!\to\!\infty$) and MPPI
($\zeta=1/\lambda$) are its limiting cases and, under SIGReg, the policy trains
against the exact EFE.

\section{Schematic Figures}
\label{app:figures}

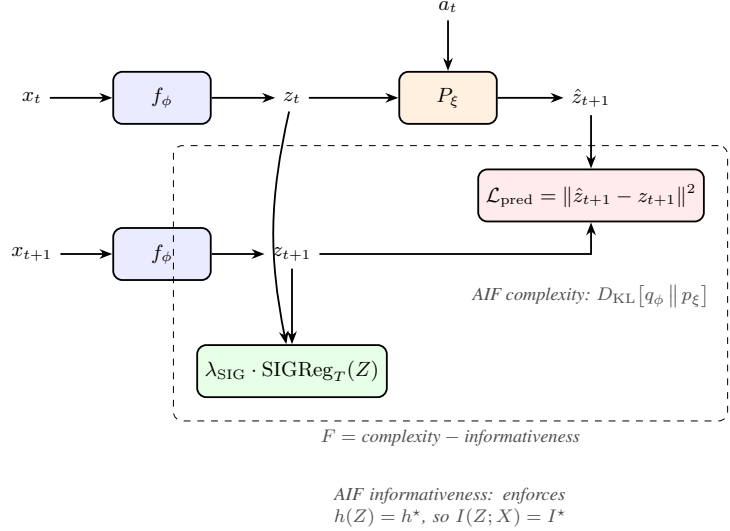
\begin{figure}[H]
\centering
\scalebox{0.86}{%
\begin{tikzpicture}[
    font=\small,
    box/.style={draw, rounded corners, minimum height=8mm, minimum width=15mm, align=center, thick},
    enc/.style={box, fill=blue!8},
    pred/.style={box, fill=orange!12},
    reg/.style={box, fill=green!10},
    arr/.style={-{Stealth[length=2mm]}, thick},
    lab/.style={font=\footnotesize\itshape, text=black!70}
  ]
  \node (ot) at (0,1.2) {$x_t$};
  \node (ot1) at (0,-1.2) {$x_{t+1}$};
  \node[enc] (enc1) at (2.0,1.2) {$f_\phi$};
  \node[enc] (enc2) at (2.0,-1.2) {$f_\phi$};
  \node (zt) at (4.0,1.2) {$\z_t$};
  \node (zt1) at (4.0,-1.2) {$\z_{t+1}$};
  \node[pred] (P) at (6.4,1.2) {$P_\xi$};
  \node (zhat) at (8.6,1.2) {$\hat\z_{t+1}$};
  \node (a) at (6.4,2.6) {$a_t$};
  \node[box, fill=red!8] (Lpred) at (8.6,-0.3) {$\Lpred=\|\hat\z_{t+1}-\z_{t+1}\|^2$};
  \node[reg] (reg) at (4.0,-3.0) {$\lambda_{\mathrm{SIG}}\cdot\Sig_T(\Z)$};
  \draw[arr] (ot) -- (enc1);
  \draw[arr] (ot1) -- (enc2);
  \draw[arr] (enc1) -- (zt);
  \draw[arr] (enc2) -- (zt1);
  \draw[arr] (zt) -- (P);
  \draw[arr] (a) -- (P);
  \draw[arr] (P) -- (zhat);
  \draw[arr] (zhat) -- (Lpred);
  \draw[arr] (zt1) -| (Lpred);
  \draw[arr] (zt) to[bend right=12] (reg);
  \draw[arr] (zt1) -- (reg);
  \node[draw, dashed, rounded corners, fit=(Lpred)(reg), inner sep=10pt,
        label={[lab]below:{$F = \text{complexity} - \text{informativeness}$}}] (fbox) {};
  \node[lab, anchor=north] at (8.6,-1.55) {AIF complexity: $\KL{q_\phi}{p_\xi}$};
  \node[lab, text width=3.6cm, align=center, anchor=north, yshift=-9mm] at (fbox.south)
    {AIF informativeness: enforces $h(\Z)=\hstar$, so $I(\Z;X)=\Istar$};
\end{tikzpicture}}
\caption{The JEPA world-model loss read as an AIF variational free energy
(main-text Theorem~\ref{thm:correspondence}). The next-embedding prediction loss
$\Lpred$ is, up to the $1/(2\sigma^2)$ scale and the additive constant
$C_{\mathrm{KL}}$ of \eqref{eq:bridge}, the AIF complexity term (a KL divergence
under the Gaussian bridge); the regulariser is the AIF informativeness term. Under SIGReg the informativeness
is \emph{enforced} rather than estimated, so both terms are exact and the total
is a valid free energy. Swapping SIGReg for VICReg changes only the regulariser
box, but reopens the prior-miscalibration gap. The ``exact free energy'' reading
holds only under the constant-noise encoder model (Fact~2) and successful SIGReg
enforcement ($p_\Z=\Norm{0}{\tfrac cd I_d}$ in the population limit); away from
these conditions it degrades gracefully (Corollary~\ref{cor:graceful}).}
\label{app:fig:architecture}
\end{figure}

\begin{figure}[H]
\centering
\begin{tikzpicture}[font=\small]
  \draw[-{Latex[length=2.4mm]},thick] (0,0) -- (12,0);
  \node[anchor=west] at (12.05,0) {entropy};
  \def\xPD{1.4}\def\xtrue{4.2}\def\xLD{7.0}\def\xVR{9.8}
  \foreach \x in {\xPD,\xtrue,\xLD,\xVR}{\draw[thick] (\x,-0.12) -- (\x,0.12);}
  \node[align=center,anchor=north,text=blue!55!black] at (\xPD,-0.2)
     {$\hPD(\Z)$\\[-1pt]\scriptsize PairDist\\[-2pt]\scriptsize(lower bound, safe)};
  \node[align=center,anchor=north] at (\xtrue,-0.2)
     {$h(\Z)$\\[-1pt]\scriptsize true entropy};
  \node[align=center,anchor=north] at (\xLD,-0.2)
     {$\hN(\Z)$\\[-1pt]\scriptsize LogDet\\[-2pt]\scriptsize(tight upper bound)};
  \node[align=center,anchor=north] at (\xVR,-0.2)
     {$\hVR(\Z)$\\[-1pt]\scriptsize VICReg\\[-2pt]\scriptsize(loose upper bound)};
  \draw[decorate,decoration={brace,amplitude=4pt},thick]
     (\xtrue,0.35) -- (\xLD,0.35)
     node[midway,above,yshift=2pt,align=center]{\scriptsize Gap I\\[-2pt]\scriptsize(non-Gaussianity)};
  \draw[decorate,decoration={brace,amplitude=4pt},thick]
     (\xLD,0.35) -- (\xVR,0.35)
     node[midway,above,yshift=2pt,align=center]{\scriptsize Gap II\\[-2pt]\scriptsize(off-diagonal)};
  \draw[-{Latex[length=2.4mm]},very thick,red!65!black]
     (\xVR,2.15) .. controls (8.2,2.75) and (5.6,2.75) .. (\xtrue,2.05);
  \node[anchor=south,text=red!65!black,align=center] at (7.0,2.95)
     {\scriptsize SIGReg: distributional enforcement\\[-2pt]\scriptsize collapses all gaps ($\Delta=0$)};
  \node[align=center,anchor=south,text=red!65!black] at (\xtrue,1.5)
     {\scriptsize $h(\Z)=\hstar(c,d)$};
\end{tikzpicture}
\caption{The entropy-estimator hierarchy and the bound-type ``sandwich''
(reproduced and extended from the first iteration). PairDist lower-bounds the
true entropy $h(\Z)$ (safe to maximise); LogDet and VICReg are upper bounds
(unsafe), separated from $h(\Z)$ by the non-Gaussianity gap (Gap~I) and from each
other by the off-diagonal-covariance gap (Gap~II). SIGReg enforces the isotropic
Gaussian directly, driving $h(\Z)\!\to\!\hstar(c,d)$ and collapsing all gaps
($\Delta\!=\!0$; main-text Proposition~\ref{prop:sigzero}).}
\label{app:fig:hierarchy}
\end{figure}
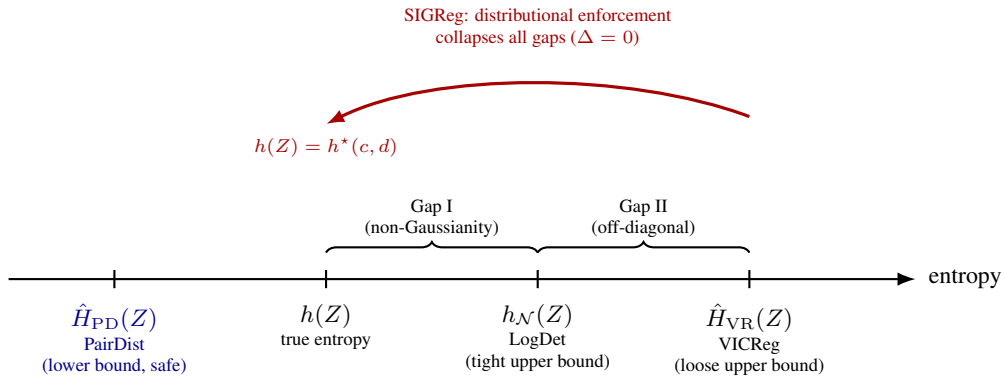

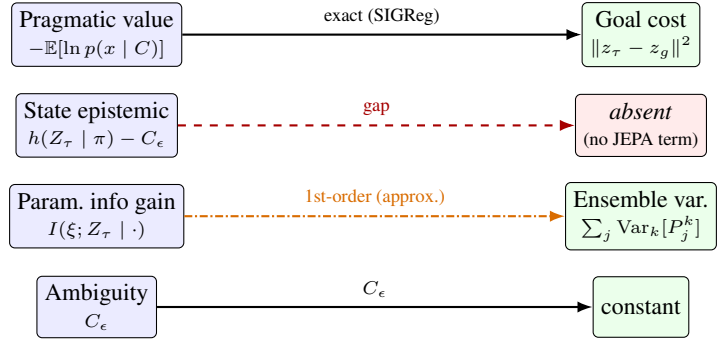
\begin{figure}[H]
\centering
\begin{tikzpicture}[font=\small,
   box/.style={draw,rounded corners=2pt,align=center,inner sep=3pt,minimum height=8mm},
   aif/.style={box,fill=blue!8},
   jepa/.style={box,fill=green!8},
   map/.style={-{Latex[length=2mm]},thick}]
  \node[aif] (prag) at (0,3.0) {Pragmatic value\\\scriptsize $-\E[\ln p(x\mid C)]$};
  \node[aif] (epi)  at (0,1.8) {State epistemic\\\scriptsize $h(Z_\tau\mid\pi)-\Ce$};
  \node[aif] (param)at (0,0.6) {Param.\ info gain\\\scriptsize $I(\xi;Z_\tau\mid\cdot)$};
  \node[aif] (amb)  at (0,-0.6){Ambiguity\\\scriptsize $\Ce$};
  \node[jepa] (goal) at (7.2,3.0) {Goal cost\\\scriptsize $\|z_\tau-z_g\|^2$};
  \node[jepa,fill=red!8] (none) at (7.2,1.8) {\emph{absent}\\\scriptsize (no JEPA term)};
  \node[jepa] (unc)  at (7.2,0.6) {Ensemble var.\\\scriptsize $\sum_j\Var_k[P^k_j]$};
  \node[jepa] (const)at (7.2,-0.6){constant};
  \draw[map] (prag) -- (goal) node[midway,above,align=center]{\scriptsize exact (SIGReg)};
  \draw[map,dashed,red!65!black] (epi) -- (none) node[midway,above]{\scriptsize gap};
  \draw[map,densely dashdotted,orange!85!black] (param) -- (unc) node[midway,above]{\scriptsize 1st-order (approx.)};
  \draw[map] (amb) -- (const) node[midway,above]{\scriptsize $\Ce$};
\end{tikzpicture}
\caption{Term-by-term correspondence between the AIF expected free energy (left)
and the JEPA world-model planning cost (right), as established in
\S\ref{app:efe}. Under SIGReg the pragmatic value maps exactly to the latent goal
cost (Proposition~\ref{app:prop:pragmatic}); the parameter information gain maps
to ensemble variance at first order; ambiguity is a constant; and the state
epistemic value (coverage) has no counterpart in any current JEPA world model:
the primary structural gap (Proposition~\ref{prop:o2}). The arrow style encodes
the nature of each correspondence: solid black for \emph{exact} mappings (under
SIGReg), amber dash-dot for the \emph{first-order approximate} mapping, and dashed
red into the shaded box for the \emph{absent} term.}
\label{app:fig:efe}
\end{figure}

\begin{figure}[H]
\centering
\begin{tikzpicture}[font=\small,
   cell/.style={draw,minimum width=29mm,minimum height=12mm,align=center,inner sep=2pt}]
  \node at (0,2.05) {};
  \node[anchor=south] at (1.5,1.55) {\scriptsize Gaussian};
  \node[anchor=south] at (1.5,1.25) {\scriptsize not assumed};
  \node[anchor=south] at (4.5,1.55) {\scriptsize Gaussian};
  \node[anchor=south] at (4.5,1.25) {\scriptsize assumed};
  \node[anchor=south] at (7.5,1.55) {\scriptsize Gaussian};
  \node[anchor=south] at (7.5,1.25) {\scriptsize enforced};
  \node[anchor=east,align=right] (rowA) at (-0.15,0.55) {\scriptsize diagonal $\Sigma$};
  \node[anchor=east,align=right] (rowB) at (-0.15,-0.65){\scriptsize full / isotropic $\Sigma$};
  \node[cell] (c11) at (1.5,0.55) {---};
  \node[cell] (c12) at (4.5,0.55) {VICReg\\\scriptsize (Level 0, UB)};
  \node[cell] (c13) at (7.5,0.55) {---};
  \node[cell] (c21) at (1.5,-0.65) {PairDist\\\scriptsize (Level 2, LB)};
  \node[cell] (c22) at (4.5,-0.65) {LogDet\\\scriptsize (Level 1, UB)};
  \node[cell,fill=red!8] (c23) at (7.5,-0.65) {\textbf{SIGReg}\\\scriptsize (Level 3, exact)};
  \draw[-{Latex[length=2mm]}] (-0.1,1.2) -- (-0.1,-1.25);
  \node[fit=(rowA)(rowB),inner sep=0pt] (rowlabels) {};
  \node[rotate=90,anchor=south] at ([xshift=-2.5mm]rowlabels.west)
     {\scriptsize covariance structure};
  \draw[-{Latex[length=2mm]}] (0.1,1.95) -- (8.9,1.95);
  \node[anchor=west] at (8.95,1.95) {\scriptsize shape enforcement};
\end{tikzpicture}
\caption{The two-dimensional design space of non-contrastive entropy
regularisers: covariance structure (diagonal $\to$ full/isotropic) versus
distributional-shape enforcement (assumed $\to$ enforced). VICReg, LogDet, and
PairDist each occupy one cell; SIGReg is the unique estimator that both enforces
the isotropic covariance and tests the full distributional shape, closing all
three gap sources of Proposition~\ref{prop:gap}. ``UB''/``LB'' denote upper/lower
entropy bounds.}
\label{app:fig:design}
\end{figure}
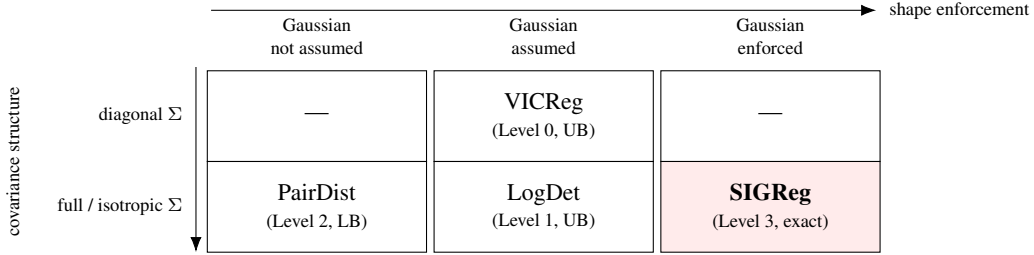

\section{Positioning and Relevance}
\label{app:positioning}

This appendix expands two points that the main text states only briefly: how the
contribution differs from the closely related JEPA literature
(\S\ref{app:differentiation}), and why the active-inference perspective it
introduces is worth having (\S\ref{app:relevance}).

\subsection{Differentiation from the JEPA programme}
\label{app:differentiation}

Because the JEPA line of work is closely related and shares authorship, it is
worth stating precisely how this paper differs from each prior result. The
contribution is orthogonal to each, a different \emph{kind} of statement about
the same systems, rather than an increment on any of them.

\paragraph{The architecture and design programme.}
The Joint-Embedding Predictive Architecture and the path-to-autonomy programme
\citep{lecun2022path} are an \emph{engineering} proposal: they specify how to
build a latent world model that learns from observation and plans without
pixel-level reconstruction, and they justify the design by what it enables. This
paper supplies the \emph{normative} account those designs lacked. It identifies
the principle, minimisation of an Active Inference variational free energy,
that a SIGReg-regularised JEPA implicitly optimises, and thereby explains
\emph{why} the architecture has the form it does rather than only \emph{that} it
works. The two are complementary: the programme provides the object, and the
present theory provides its interpretation.

\paragraph{The regularisers.}
VICReg \citep{bardes2022vicreg} and LeJEPA/SIGReg \citep{balestriero2025lejepa}
\emph{propose} anti-collapse regularisers and justify them on self-supervised
grounds, avoidance of representational collapse and downstream task
performance. This paper \emph{evaluates} the same regularisers against an
external criterion: whether the implied entropy estimator preserves the Active
Inference free-energy bound. That criterion, the bound \emph{type} of the
estimator (Proposition~\ref{prop:boundtype}), is invisible from within the SSL
framing, and it is what distinguishes SIGReg from VICReg in kind rather than
degree. The hierarchy of Section~\ref{sec:hierarchy} is, in effect, a reading of
the regulariser landscape through a lens the original proposals did not use.

\paragraph{The optimum-characterising results.}
The Gaussian-Embeddings result \citep{balestriero2025gaussian} shows that a
successfully trained JEPA \emph{is} approximately Gaussian, a property of the
optimum. The linear-identifiability result \citep{klindt2026lejepa} establishes
\emph{when} the trained encoder recovers the world's latent up to a linear
transform, the state-representation precondition of
Section~\ref{sec:identifiability}. This paper uses both differently. It uses the
\emph{signed} gap from exact Gaussianity (not the fact of approximate
Gaussianity) to make a claim about the AIF bound and about the training
\emph{path}, not only the optimum; and it consumes the identifiability guarantee
as an upstream precondition for a \emph{planning and control} correspondence that
the identifiability authors explicitly leave open. The three results triangulate
the same isotropic-Gaussian target from maximum-entropy, density-estimation, and
spectral-identifiability directions respectively, which is a positioning strength
rather than a redundancy.

\paragraph{The world models and planners.}
LeWorldModel \citep{maes2026leworldmodel}, DINO-WM \citep{zhou2025dinowm}, PLDM
\citep{sobal2025closing}, V-JEPA~2 \citep{assran2025vjepa2}, and the planning
ablations of \citet{terver2025drives} \emph{build and benchmark} JEPA world
models and their planners. This paper shows that the components those systems
already use (the latent goal cost, the ensemble-disagreement penalty, the
model-predictive planner) \emph{are} specific Active Inference quantities
(Propositions~\ref{prop:pragmatic}--\ref{prop:o2}): the goal cost is the
pragmatic value, the ensemble disagreement is the parameter information gain, and
MPC with one-step replanning realises the exact executed-trajectory
correspondence. It further identifies the one AIF term none of these systems
computes, the state-epistemic value, and converts the whole mapping into
falsifiable predictions. In one line: the JEPA programme asks \emph{how to build}
world models that work; this paper asks \emph{why} a particular regulariser makes
them Active Inference agents, and what that implies for what they should compute.

\subsection{Relevance and potential impact of the active-inference perspective}
\label{app:relevance}

A correspondence theorem invites a fair question: granting that the mathematics
holds, what does the active-inference reading actually buy? We give four answers,
ordered from the most concrete to the most speculative, and flag the last two as
conjectural.

\paragraph{A normative selection criterion for regularisers.}
The most immediate consequence is practical. Absent a normative principle, the
choice among anti-collapse regularisers is made on empirical grounds that do not
transfer across settings. The AIF reading supplies a criterion that does
transfer: a regulariser should preserve the free-energy bound, which requires a
lower-bound (or exact) entropy estimator rather than an upper-bound one
(Proposition~\ref{prop:boundtype}). This both explains an existing empirical
preference for SIGReg and predicts the behaviour of regularisers not yet
proposed, by locating them on the two-dimensional design space of
Figure~\ref{app:fig:design}. The criterion is useful independently of whether one
accepts the full AIF framing, because it is ultimately a statement about
information-theoretic safety.

\paragraph{A concrete, missing architectural component.}
The decomposition identifies a specific term, the state-epistemic value, the
entropy of the predicted future-state distribution under a policy
(Proposition~\ref{prop:o2}), that no current JEPA world model computes. This is
not a philosophical observation but an engineering target: it names a coverage
signal that, if added to the planning objective, would give a JEPA agent a
principled drive to explore, of exactly the kind that distinguishes
active-inference agents from purely goal-reaching controllers. The perspective
thus does more than re-describe existing systems; it points at a buildable
addition and predicts its effect.

\paragraph{A bridge that lets two literatures share results.}
Establishing that JEPA world models and AIF agents optimise the same functional
means that results proven in one framework become available to the other. The
neuroscience and theoretical-AIF literature has developed an extensive account of
exploration, precision-weighting, hierarchical generative models, and the role of
the habit prior; the SSL and world-model literature has developed scalable
encoders, stable training, and efficient planners. A formal correspondence is the
precondition for transfer in both directions, for importing AIF's exploration
machinery into scalable world models, and for giving AIF's normative claims a
concrete, large-scale computational instantiation. The learned-policy regime
(Proposition~\ref{prop:policy}), which casts action selection as amortised
inference, is one such transfer already visible within the paper.

\paragraph{A conjectural broader significance.}
More speculatively, and we mark this clearly as conjecture rather than
established consequence, the correspondence bears on a longstanding question
about whether the Free Energy Principle, often criticised as unfalsifiable in its
most general form, can be given empirical teeth in artificial systems. By tying a
specific FEP-derived quantity (the surprise bound) to a measurable property of a
specific class of trainable models (the embedding condition number, the
latent-to-physical calibration), the present work turns a portion of the
principle into something a particular experiment can confirm or refute. Whether
this generalises beyond the constant-noise, isotropic-Gaussian regime studied
here is open, and we do not claim that a successful test of these predictions
would vindicate the Free Energy Principle as a whole. The narrower claim is that
the AIF perspective, applied to JEPA world models, is productive in the
scientifically meaningful sense: it generates predictions that could be wrong.

\section{Lean Verification}
\label{app:lean}

Following the precedent of \citet{klindt2026lejepa}, we formally verify the
algebraic core of this paper's results, namely the Gaussian bridge
\eqref{eq:bridge} with its optimiser-set equivalence, the three-source gap
decomposition and the bound-type dichotomy (Propositions~\ref{prop:gap}
and~\ref{prop:boundtype}, Corollary~\ref{cor:graceful}(i)), the single-step
correspondence (Proposition~\ref{prop:sigzero},
Theorem~\ref{thm:correspondence}(i)--(iii)), the exact expected-free-energy
term identities (Propositions~\ref{prop:pragmatic} and~\ref{prop:o2};
Proposition~\ref{app:prop:efe-rest}), the amortised-policy result
(Proposition~\ref{prop:policy}), and the exact parts of the multi-step
correspondence (Propositions~\ref{app:prop:tf} and~\ref{app:prop:comp},
Corollary~\ref{app:cor:mpc}, Theorem~\ref{thm:multistep}(i),(iii),(iv)), in
the Lean~4 theorem prover \citep{demoura2021lean4} using the Mathlib
mathematical library \citep{mathlib2020}. The development compiles with zero
\texttt{sorry} obligations: every logical step from the stated premises to the
stated conclusions is machine-checked.

\paragraph{Scope and methodology.}
Lean verification requires every inference step to be justified by a
previously established lemma, hypothesis, or axiom. When a standard result
exists in Mathlib (e.g.\ \texttt{Real.log\_le\_sub\_one\_of\_pos} for
$\ln x\le x-1$, \texttt{ConcaveOn.le\_map\_sum} for Jensen's inequality), we
invoke it directly. Analytic facts that Mathlib cannot yet express at our
encoding, because they concern differential entropies of distributions that
the scalar development treats as opaque real parameters, enter as \emph{named
hypotheses in the theorem statements} rather than as global axioms: the
maximum-entropy theorem behind Gap~I \citep[Thm.~8.6.5]{cover2006elements},
the isotropic-Gaussian entropy value behind enforcement
\citep[Thm.~8.4.1]{cover2006elements}, conditioning-reduces-entropy behind
Proposition~\ref{prop:o2}, and the Fact-1 surprise bound. This is a deliberate
refinement of the axiom-based convention of \citet{klindt2026lejepa}: over
opaque real-valued parameters a global axiom would be unsound, and the
hypothesis form makes each result's analytic dependencies auditable in its own
signature. Exactly one classical result is introduced as a Lean \texttt{axiom},
because it is stateable about concrete Mathlib objects and is absent from
Mathlib: Hadamard's determinant inequality for positive-semidefinite real
matrices \citep[Thm.~7.8.1]{horn2012matrix}, consumed only by the matrix form
of Gap~II in Proposition~\ref{prop:gap}. The complete reasoning chains between
these inputs are fully verified; Table~\ref{tab:lean} gives the inventory. One
scoping rule is worth stating: the finite-sample rate
(Proposition~\ref{prop:rate}) is excluded entirely, because its quantitative
Cram\'er--Wold step is classical only in qualitative form
(Remark~\ref{rem:rate}), and axiomatising unsettled mathematics would
counterfeit the standard under which axioms stand for settled results; the
Taylor-remainder analyses (Proposition~\ref{app:prop:jac},
Theorem~\ref{thm:multistep}(ii)) are likewise deferred rather than
axiomatised.

\paragraph{What is verified.}
All reasoning between the stated inputs is machine-checked: the bridge
identity, the nonnegativity of its constant $C_{\mathrm{KL}}$, and the
optimiser-set equivalence between KL and squared error; the three-source
decomposition and, from the Hadamard axiom, the sign of Gap~II; the bound-type
dichotomy of Proposition~\ref{prop:boundtype} and the exact linear slack of
Corollary~\ref{cor:graceful}(i); Gibbs' inequality, the KL-completion identity,
and the Boltzmann-optimality of Proposition~\ref{prop:policy} in full; the
teacher-forced summation of Proposition~\ref{app:prop:tf}; the unrolled
Lipschitz recurrence and the planning-MSE decomposition of
Proposition~\ref{app:prop:comp}; the executed-trajectory specialisation of
Corollary~\ref{app:cor:mpc}; the ranking-fidelity equivalence of
Theorem~\ref{thm:multistep}(iv); and the Jensen argument that isotropy
maximises the Gaussian log-volume at fixed total variance
(Proposition~\ref{prop:sigzero}). The correspondence statements themselves
(Proposition~\ref{prop:sigzero}, Theorem~\ref{thm:correspondence}(i)--(iii),
Proposition~\ref{prop:o2}, Theorem~\ref{thm:multistep}(i),(iii)) are verified
as \emph{assemblies}: their algebra is machine-checked, with successful SIGReg
enforcement entering as an explicit hypothesis, matching the scoping of
Remark~\ref{rem:scope}. The formalisation also surfaced an additive-constant
bookkeeping discrepancy between an earlier draft of
Definition~\ref{def:sigfe} and the per-step free energy of
Appendix~\ref{app:proofs}; Definition~\ref{def:sigfe} as stated incorporates
the correction, and the offset identity is itself a machine-checked lemma
(\texttt{mainText\_constant\_offset}).

\begin{table}[h]
\centering
\footnotesize
\caption{Lean verification inventory. \textsc{verified}: machine-checked from
Mathlib primitives with no analytic hypotheses. \textsc{assembly}: algebraic
derivation machine-checked, analytic inputs (e.g.\ SIGReg enforcement, the
Fact-1 bound) as named hypotheses. The Hadamard axiom is the only Lean
\texttt{axiom} in the development.}
\label{tab:lean}
\setlength{\tabcolsep}{3pt}
\resizebox{\linewidth}{!}{%
\begin{tabular}{@{}lll@{}}
\toprule
Result & Lean declaration & Status \\
\midrule
Gaussian bridge \eqref{eq:bridge}; $C_{\mathrm{KL}}\ge0$; argmin equiv. & \texttt{gaussKL\_eq}, \texttt{gaussKL\_le\_iff} & \textsc{verified} \\
Prop.~\ref{prop:gap} decomposition; Gap II $\le0$ & \texttt{gap\_decomposition}, \texttt{gapII\_nonpos} & \textsc{verified}$^{\dagger}$ \\
Prop.~\ref{prop:boundtype} (bound type) & \texttt{boundtype\_*} & \textsc{verified} \\
Cor.~\ref{cor:graceful}(i) (linear slack) & \texttt{graceful\_degradation\_abs} & \textsc{verified} \\
Prop.~\ref{prop:sigzero}; isotropy maximises & \texttt{sigreg\_zero\_gap}; \texttt{isotropy\_maximises} & \textsc{assembly}; \textsc{verified} \\
Thm.~\ref{thm:correspondence}(i)--(iii) & \texttt{correspondence\_single\_step} & \textsc{assembly} \\
Prop.~\ref{prop:pragmatic}/\ref{app:prop:pragmatic} (exact KL proxy) & \texttt{pragmatic\_exact} & \textsc{verified} \\
Prop.~\ref{prop:o2} (state-epistemic ceiling) & \texttt{state\_epistemic\_bound} & \textsc{assembly} \\
Prop.~\ref{prop:policy} (KL completion; argmin) & \texttt{policy\_kl\_completion}, \texttt{policy\_argmin} & \textsc{verified} \\
Prop.~\ref{app:prop:tf} (teacher-forced KL--MSE) & \texttt{teacherForced\_kl\_mse} & \textsc{verified} \\
Prop.~\ref{app:prop:comp} (compounding; plan-MSE decomp.) & \texttt{compounding\_error\_bound} & \textsc{verified} \\
Cor.~\ref{app:cor:mpc} ($m{=}1$ exactness) & \texttt{executed\_trajectory\_exact} & \textsc{verified} \\
Thm.~\ref{thm:multistep}(i),(iii); (iv) ranking & \texttt{multistep\_correspondence}; \texttt{ranking\_fidelity} & \textsc{assembly}; \textsc{verified} \\
Prop.~\ref{app:prop:efe-rest} risk identity & \texttt{risk\_identity} & \textsc{verified} \\
\bottomrule
\end{tabular}%
}

\smallskip
\raggedright\footnotesize $^{\dagger}$The matrix form of Gap~II consumes the
declared Hadamard axiom; all log algebra is machine-checked.
\end{table}

\paragraph{Build information.}
The project comprises 37 Lean theorem/lemma declarations, 12 definitions, and
1 axiom across six modules (746 lines of Lean). These are Lean declarations,
which do not map one-to-one onto the paper's numbered results: Lean does not
distinguish proposition or corollary from theorem, and a single result
typically expands into several declarations (one per clause), together with
supporting lemmas that the prose folds away (e.g.\ Gibbs' inequality behind
Proposition~\ref{prop:policy}, or the Boltzmann normalisation facts behind its
argmin corollary). The development compiles against Lean~4 v4.32.0 with
Mathlib v4.32.0 (8{,}662 build targets, zero errors, zero \texttt{sorry}
obligations).
A per-theorem \texttt{\#print axioms} audit confirms that every declaration
depends only on Mathlib's three standard axioms (\texttt{propext},
\texttt{Classical.choice}, \texttt{Quot.sound}), except the matrix form of
Gap~II, which additionally depends on the declared Hadamard axiom. The
development is publicly available at
\url{https://github.com/FabioArnez/sigreg-vfe-correspondence-lean}.

\end{document}